\documentclass{template/colt2021} 

\usepackage{newtxtext}
\usepackage{caption}
\usepackage{tikz}
\usepackage{pgfplots}
\pgfplotsset{compat=1.17}
\usepackage{microtype}
\usepackage{enumitem}
\setitemize{itemsep=0pt,parsep=0pt,topsep=0pt}
\setenumerate{itemsep=0pt,parsep=0pt,topsep=0pt}
\usepackage{xfrac}

\newcommand{\card}[1]{\# #1 }
\newcommand{\diff}{\mathop{}\!\mathrm{d}}
\newcommand{\ind}[1]{\mathbf{1}_{#1}}
\newcommand{\prob}[1]{\Delta_{#1}}
\DeclareMathOperator{\E}{\mathbb{E}}
\DeclareMathOperator{\Pbb}{\mathbb{P}}

\DeclareMathOperator*{\argmin}{arg\,min}

\DeclareMathOperator{\hull}{Conv}
\DeclareMathOperator{\ima}{im}
\DeclareMathOperator{\op}{op}
\DeclareMathOperator{\sign}{sign}
\DeclareMathOperator*{\supp}{supp}
\DeclareMathOperator{\trace}{Tr}

\newcommand{\N}{\mathbb{N}}

\newcommand{\R}{\mathbb{R}}

\renewcommand{\brace}[1]{\left\{ #1 \right\}}
\newcommand{\bracket}[1]{\left[ #1 \right]}
\newcommand{\floor}[1]{\left\lfloor #1 \right\rfloor}

\newcommand{\paren}[1]{\left( #1 \right)}
\newcommand{\midvert}{\,\middle\vert\,}

\newcommand{\abs}[1]{\left| #1 \right|}
\newcommand{\norm}[1]{\left\| #1 \right\|}
\newcommand{\scap}[2]{\left\langle #1, #2 \right\rangle}

\newcommand{\X}{\mathcal{X}}
\newcommand{\Y}{\mathcal{Y}}

\renewcommand{\epsilon}{\varepsilon}
\renewcommand{\phi}{\varphi}

\newtheorem{assumption}{Assumption}

\newcommand{\myfunction}[5]{
\[
\begin{array}{cccc}
	#1 : & #2 & \rightarrow & #3 \\
	 & #4 & \rightarrow & #5
\end{array}
\]
}

\title[Fast Rates for Structured Prediction]{Fast Rates for Structured Prediction}

\coltauthor{%
 \Name{Vivien Cabannes} \Email{vivien.cabannes@gmail.com}\\
 \Name{Francis Bach} \Email{francis.bach@gmail.com}\\
 \Name{Alessandro Rudi} \Email{alessandro.rudi@gmail.com}\\
 \addr INRIA - D\'epartement d'Informatique de l'\'Ecole
 Normale Sup\'erieure - PSL Research University, Paris, France%
}

\begin{document}

\maketitle

\begin{abstract}%
  Discrete supervised learning problems such as classification are often tackled by introducing a continuous surrogate problem akin to regression. Bounding the original error, between estimate and solution, by the surrogate error endows discrete problems with convergence rates already shown for continuous instances. Yet, current approaches do not leverage the fact that discrete problems are essentially predicting a discrete output  when continuous problems are predicting a continuous value.
  In this paper, we tackle this issue for general structured prediction problems, opening the way to ``super fast'' rates, that is, convergence rates for the excess risk faster than $n^{-1}$, where $n$ is the number of observations, with even exponential rates with the strongest assumptions. We first illustrate it for predictors based on nearest neighbors, generalizing rates known for binary classification to any discrete problem within the framework of structured prediction. We then  consider kernel ridge regression where we improve known rates in $n^{-1/4}$ to arbitrarily fast rates, depending on a parameter characterizing the hardness of the problem, thus allowing, under smoothness assumptions, to bypass the curse of dimensionality.
\end{abstract}


\section{Introduction}

Machine learning is raising high hopes to tackle a wide variety of prediction problems, such as
 language translation, fraud detection, traffic routing, speech recognition,
self-driving cars, DNA-binding proteins, {\em etc.}. 
Its framework is appreciated as it removes humans from the burden to come up with a set
of precise rules to accomplish a complex task, such as recognizing a cat on an array of pixels. 
Yet, it comes at a price, which is of forgetting about algorithm correctness,
meaning that machine learning algorithms can make mistakes, {\em i.e.}, wrong predictions, which can have
dramatic implications, {\em e.g.}, in medical applications.
This motivates work on generalization error bounds, quantifying how often one
should expect errors. 

Many of the problems discussed above are of discrete nature, in the sense
that the number of potential outputs is finite, or infinite countable.
To learn such problems, a classical technique consists in defining a continuous
surrogate problem, which is easier to solve, and such that: 
\begin{itemize}
\item[(1)] an algorithm on the surrogate problem translates into an algorithm on the original problem;
\item[(2)] errors on the original problem are bounded by errors on the surrogate problem.
\end{itemize}
The first point refers to the concept of plug-in algorithms, while the second
point to the notion of calibration inequalities.
For example, binary classification can be approached through regression by
estimating the conditional expectation of the output $Y$ given an input $X$ \citep{Bartlett2006}. 

On the one hand, continuous surrogates for discrete problems are interesting, as they
benefit from functional analysis knowledge, when discrete problems are more
combinatorial in nature.
On the other hand, continuous surrogate can be deceptive, as they are asking to solve
for more than needed. Considering the example of binary classification, where
$Y\in\brace{-1, 1}$, one only has to predict the sign of the conditional
expectation, rather than its precise value. 
Interestingly, without modifying the continuous surrogate approach, this last remark
can be leveraged in order to tighten generalization bounds derived through
calibration inequalities \citep{Audibert2007}.
In this work, we extend those considerations, known 
in binary classification \citep[\emph{e.g.,}][]{Koltchinskii2005, Chaudhuri2014},
to generic discrete supervised learning problems, and 
show how it can be applied to the kernel ridge regression algorithm introduced
by \citet{Ciliberto2016}. 

\subsection{Contributions}
Our contributions are organized in the following order.
\begin{itemize}
  \item In Section \ref{sec:structured}, we consider the general structured
    prediction from \citet{Ciliberto2020} 
    and derive refined calibration 
    inequalities to leverage the fact that learning a mapping into a
    discrete output space is easier than learning a mapping into a continuous space.
  \item In Section \ref{sec:margin}, we show how to exploit exponential concentration
    inequalities to turn them into fast rates under a condition generalizing the Tsybakov
    margin condition.
  \item In Section \ref{sec:nn}, we apply Section \ref{sec:margin}
    to local averaging methods with the particular example of nearest
    neighbors. This leads to extending the rates known for regression and
    classification to a wide variety of structured prediction problems, with
    rates that match minimax rates known in binary classification.
  \item In Section \ref{sec:rkhs}, we show how Section \ref{sec:margin} can be
    applied to kernel ridge regression.
    This allows us to improve rates known in $n^{-1/4}$ to arbitrarily fast rates
    depending on the hardness of the associated discrete problem.
\end{itemize}

\subsection{Related work}
 
\paragraph{Surrogate framework.}
The surrogate problem we will consider to tackle structured prediction finds its roots
in the approximate Bayes rule proposed by 
\citet{Stone1977}, analyzed through the prism of mean estimation as suggested by
\citet{Friedman1994} for classification, and  analyzed by
\citet{Ciliberto2020} in the wide context of structured prediction.
In particular, we will specify results on two classes of surrogate estimators: local averaging methods, or
 kernel ridge regression.

\paragraph{Local averaging methods.}
Neighborhood methods were first studied by \citet{Fix1951} for statistical testing
through density estimation. Similarly  Parzen–Rosenblatt window
methods \citep{Parzen1962,Rosenblatt1956} were developed. Those methods were cast in the
context of regression as nearest neighbors \citep{Cover1957} and Nadayara-Watson
estimators \citep{Watson1964,Nadaraya1964}.
\citet{Stone1977} was the first to derive consistency results for a large class
of localized methods, among which are nearest neighbors and some window
estimators \citep{Spiegelman1980,Devroye1980}.
Rates were then derived, with minimax optimality \citep{Stone1980,Yang1999}.
Several reviews can be found in the literature, such as \citet{Gyorfi2002,Tsybakov2009,Biau2015,Chen2018}.

\paragraph{Reproducing kernel ridge regression.}
The theory of real-valued reproducing kernel Hilbert spaces was formalized by
\citet{Aronszajn1950}, before finding applications in machine learning \citep[\emph{e.g.,}][]{Scholkopf2001}.
Minimax rates for kernel ridge regression were achieved by
casting the empirical solution estimate as a result of integral operator approximation
\citep{Smale2007,Caponnetto2007}, allowing to control convergence through
concentration inequalities in Hilbert spaces \citep{Yurinskii1970,Pinelis1986}
and on self-adjoint operating on Hilbert spaces \citep{Minsker2017}.
First derived in $L^{2}$-norm, rates were cast in $L^{\infty}$-norm through interpolation
inequalities \citep[\emph{e.g.,}][]{Fischer2020,Lin2020}.

\paragraph{Tsybakov margin condition.}
Learning a mapping into a discrete output space is indeed easier than learning a
continuous mapping, as, for binary classification for example, one typically only has to predict the sign of 
$\E[Y\vert X]$ rather than its precise value.
As such, calibration inequalities that relate the error on a discrete structured
prediction problem to an error on a smooth surrogate problem are often suboptimal.
This phenomenon was exploited for density discrimination, a problem consisting of testing if samples were drawn from one or the other of two potential distributions, by
\citet{Mammen1999}, and for binary classification by \citet{Audibert2007}.
Those works introduce a parameter $\alpha\in[0, \infty)$ characterizing the
hardness of the discrete problem, and leverage concentration inequalities to
accelerate rates known for regression by a power $\alpha+1$
\citep{Audibert2007}, while rates plugged-in directly through calibration
inequalities only present an acceleration by a power
$\sfrac{2(\alpha+1)}{(\alpha + 2)}$
\citep[\emph{see, e.g.,}][]{Boucheron2005,Bartlett2006,Bartlett2006b,Ervan2015,Nowak2019}.

\section{Structured Prediction with Surrogate Control}
\label{sec:structured}

In this section, we introduce the classical supervised learning problem, and a
surrogate problem that consists of conditional mean estimation.
We recall a calibration inequality relating the original problem to the
surrogate one.
We mention how empirical estimations of the conditional means usually deviate from
the real means following a sub-exponential tail bound, similarly to bounds
obtained through Bernstein inequality.
We end this section by providing refined surrogate control, that is the key
towards ``super fast'' rates, that is, rates faster than $1/n$.

\subsection{Surrogate mean estimation}
Consider a classic supervised learning problem, where given an input space
$\X$, an observation space~$\Y$, a prediction space $\cal Z$, a joint
distribution $\rho \in \prob{\X\times\Y}$ and a loss function
$\ell:{\cal Z}\times\Y\to\R_{+}$, one would like to retrieve
$f^{*}:\X\to{\cal Z}$ minimizing the risk $\cal R$. 
\[
  f^{*} \in \argmin_{f:\X\to{\cal Z}} {\cal R}(f)
  \qquad\text{with}\qquad
  {\cal R}(f) = \E_{(X, Y)\sim\rho}\bracket{\ell(f(X), Y)}.
\]
In practice, $\X$, $\Y$, ${\cal Z}$ and $\ell$ are givens of the problem, while
$\rho$ is unknown, yet partially observed thanks to a dataset
${\cal D}_{n} = (X_{i}, Y_{i})_{i\leq n} \sim \rho^{\otimes n}$, with data
$(X_{i}, Y_{i})$ sampled independently from $\rho$.
Note that in fully supervised learning, the observation space is the same as the
prediction space $\Y = {\cal Z}$, yet we distinguish the two for our results to
stand in more generic settings, such as instances of weak supervision
\citep{Cabannes2020}.
In the following, we consider ${\cal Z}$ finite.
In several cases, solving the supervised learning problem can be done through
solving a surrogate problem that is easier to handle.
\cite{Ciliberto2016} provide a setup that reduces a wide variety of structured
prediction problems $(\ell, \rho)$ to a problem of mean estimation.
It works under the following assumption.
\begin{assumption}[Bilinear loss decomposition]\label{ass:loss}
  There exists an Hilbert space ${\cal H}$ and two mappings
  $\psi:{\cal Z}\to{\cal H}$, $\phi:\Y\to{\cal H}$ such that
  \[
    \ell(z, y) = \scap{\psi(z)}{\phi(y)}.
  \]
  We will also assume that $\psi$ is bounded (in norm) by a constant $c_\psi$.
\end{assumption}
This assumption is not really restrictive \citep{Ciliberto2020}. Among others, it works for any losses
on finite spaces, usually with spaces $\cal H$ whose dimensionality is only polylogarithmic with respect to the cardinality of $\cal Z$ \citep{Nowak2019}. 
Under Assumption \ref{ass:loss}, solving the supervised learning problem can be
done through estimating the surrogate conditional mean $g^{*}:\supp\rho_{\X}\to{\cal H}$, defined as
\begin{equation}
\label{eq:gast}
  g^{*}(x) = \E_{Y\sim\rho\vert_{x}}\bracket{\phi(Y)},
\end{equation}
where we denote $\rho\vert_{x}$ the conditional law of $\paren{Y\midvert X}$ under
$(X, Y) \sim \rho$.

\begin{lemma}[\citet{Ciliberto2016}]\label{lem:cal}
  Given an estimate $g_{n}$ of $g^{*}$ in Eq.~\eqref{eq:gast}, consider the estimate
  $f_{n}:\X\to{\cal Z}$ of $f^{*}$, which is obtained from ``decoding'' $g_{n}$ as
  \begin{equation}\label{eq:decoding}
    f_{n}(x) = \argmin_{z\in{\cal Z}} \scap{\psi(z)}{g_{n}(x)}.
  \end{equation}
  Then the excess risk is controlled through the surrogate error as
  \begin{equation}
  \label{eq:L1}
    {\cal R}(f_{n}) - {\cal R}(f^{*}) \leq 2c_\psi
    \norm{g_{n} - g^{*}}_{L^{1}(\X, {\cal H}, \rho)}.
  \end{equation}
\end{lemma}

Inequalities relating the original excess risk
${\cal R}(f_{n}) - {\cal R}(f^{*})$ with a measure of error on a surrogate problem
are called \emph{calibration inequalities}.
They are useful when the measure of error between $g_{n}$ and $g^{*}$
is easier to control than the one between $f_{n}$ and $f^{*}$.

\begin{example}[Binary classification]\label{ex:binary}
  Binary classification corresponds to
  $\Y={\cal Z} = \brace{-1, 1}$ and $\ell(z, y) = \ind{z\neq y}$ (or
  equivalently $\ell(z, y) = 2\ind{z\neq y}-1$).
  The classical surrogate consists of taking $\cal H = \R$, with
  $\phi=\textit{id}$ and $\psi=-\textit{id}$. In this setting, we have
  $g^{*}(x) = \E_{\rho}[Y\vert X=x]$, and the decoding $f_{n}(x) := \sign g_{n}(x)$, for
  any $g_{n}(x) \in\cal H$. In this case
  \(
    {\cal R}(f_{n}) - {\cal R}(f^{*})
    = \E_{X}\bracket{\ind{f_{n}(X)\neq f^{*}(X)} \abs{g^{*}(X)}}
    \leq 2\norm{g_{n} - g^{*}}_{L^{1}}
    \leq 2\norm{g_{n} - g^{*}}_{L^{2}}.
  \)
  Note that in regression the excess risk reads as the square of the $L^{2}$
  norm, explaining a loss of a power one half in convergence rates, when going
  from regression to classification \citep[\emph{e.g.}][]{Chen2018}. 
\end{example}

Differences between an empirical estimate and its population version are
generally handled through concentration inequalities. In this work, we will
leverage concentration on $\norm{g_{n}(x) - g(x)}$ that is uniform for $x \in
\supp \rho_{\X}$, motivating the introduction of Assumption \ref{ass:concentration}.
\begin{assumption}[Exponential concentration inequality]\label{ass:concentration}
  Suppose that for $n\in\N$, there exists two reals $L_n$ and $M_n$, such that
  the tails of $\norm{g_{n}(x) - g(x)}$ can be controlled for any $t > 0$ as 
  \begin{equation}
    \label{eq:concentration}
    \sup_{x\in\supp\rho_{\X}} \Pbb_{{\cal D}_{n}}\paren{\norm{g_{n}(x) - g(x)} > t} \leq
    \exp\paren{-\frac{L_n t^{2}}{1 + M_n t}}. 
  \end{equation}
\end{assumption}
Note that to satisfy Assumption \ref{ass:concentration}, it is sufficient, yet
\emph{not necessary}, to have a uniform control on $g_{n} - g^{*}$, {\emph i.e.}, a control
on the tail of $\norm{g_{n} - g^{*}}_{L^{\infty}}$, since
\(
  \sup_{x} \Pbb\paren{{A_{x} > t}}
  \leq \Pbb\paren{\cup_{x} \brace{A_{x} > t}} = \Pbb\paren{\sup_{x} A_{x} > t},
\)
with $(A_{x})$ a family of random variables indexed by $x\in\X$.

Usually, in bounds like Eq. \eqref{eq:concentration}, $M_n$ is a constant of the
problem, while $L_n$ depends on the number of samples, therefore, we
would like to give rates depending on $L_n$. Typically in Bernstein inequalities (see
Theorem \ref{thm:bernstein-vector} in Appendix), $L_n = n\sigma^{-2}$ with $\sigma^{2}$ a
variance parameter and $M_n = c\sigma^{-2}$ with $c$ a constant of the problem
that does not depend on $n$. 

\subsection{Refined Calibration}
While it is sufficient to control the excess risk through a $L^{1}$-norm control on
$g$ from Eq.~\eqref{eq:L1}, it is not always necessary. In other words, the calibration bound in
Lemma \ref{lem:cal} is not always tight. Indeed, we do not predict optimally, that is, 
$\brace{f_{n}(x) \neq f^{*}(x)}$ only if $g_{n}(x)$ and $g^{*}(x)$ do not lead to the
same decoding $f_{n}(x)$ and $f^{*}(x)$. When $\cal Z$ is finite, this is characterized by $g_{n}(x)$ and $g^{*}(x)$  not falling in the same
region $R_z$ of ${\cal H}$, where
\[
  R_z = \big\{\xi \in {\cal H} \big| z \in \argmin_{z' \in{\cal Z}}
    \scap{\psi(z')}{\xi}\big\}.
\]
To ensure that $g_{n}(x)$ and $g^{*}(x)$ fall in the same region, one can ensure
that $g_{n}(x)$ is closer to $g^{*}(x)$ than $g^{*}(x)$ is of the frontier of those
regions.
Those frontiers are defined by points leading to at least two minimizers in Eq.~\eqref{eq:decoding}:
\begin{align*}
  F = \brace{\xi\in{\cal H} \midvert
    \big|\argmin_{z\in{\cal Z}}\scap{\psi(z)}{\xi}\big| > 1}.
\end{align*}
The introduction of $F$ is motivated by the following geometric results.
\begin{lemma}[Refined surrogate control]\label{lem:calibration}
  When $\cal Z$ is finite, for any $x\in\supp\rho_{\X}$, 
  \[
    \norm{g_{n}(x) - g^{*}(x)} < d(g^{*}(x), F) \qquad\Rightarrow\qquad f_{n}(x) = f^{*}(x),
  \]
  with $d$ the extension of the norm distance to sets as
  $d(g^{*}(x), F) = \inf_{\xi\in F} \norm{g^{*}(x) - \xi}$.
  This result allows to refine the calibration control from Lemma \ref{lem:cal} as
  \begin{equation}
    \label{eq:calibration}
    {\cal R}(f_{n}) - {\cal R}(f^{*})
    \leq 2c_{\psi} \E_{X} \bracket{ \ind{\norm{g_{n}(X) - g^{*}(X)} \geq d(g^{*}(X), F)}\norm{g_{n}(X) - g^{*}(X)}}.
  \end{equation}
\end{lemma}

\begin{example}[Binary classification]\label{ex:binary-2}
  In binary classification (\emph{cf.} Example \ref{ex:binary}), $F = \brace{0}$, and, for
  any $x\in\supp\rho_{\X}$, $d(g^{*}(x), F) = \abs{g^{*}(x)}$.
  Lemma \ref{lem:calibration} is based on the fact that
  $f^{*}(x) \neq f_{n}(x)$ implies that $\sign g^{*}(x) \neq \sign g_{n}(x)$ which
  itself implies that $\abs{g^{*}(x) - g_{n}(x)} = \abs{g^{*}(x)} + \abs{g_{n}(x)} \geq \abs{g^{*}(x)}$.
\end{example}

To leverage Eq. \eqref{eq:calibration}, we need to control $d(g^{*}(x), F)$ below and
$\norm{g_{n}(x)- g^{*}(x)}$ above. While upper bounds on $\norm{g_{n}(x)- g^{*}(x)}$ are assumed
to have been derived through concentration inequalities,
lower bounds on $d(g^{*}(x), F)$ will be assumed as a given parameter of the problem,
see Eqs.~\eqref{eq:no-density}~and~\eqref{eq:low-density}. 

\begin{remark}[Scope of our work]
  While we derived the refined calibration inequality Eq. \eqref{eq:calibration}
  for the surrogate conditional mean $g^*$ and the associated pointwise metric
  $\norm{\cdot}_{\cal H}$, similar inequality could be obtained for other type
  of surrogate methods.
  This suggests that our work could be extended to any smooth surrogate such as
  the ones considered by \citet{Nowak2019b}, as well as Fenchel-Young losses
  \citep{Blondel2020}.
\end{remark}

\subsection{Geometric understanding}

In this subsection, we detail how to understand geometrically Lemma
\ref{lem:calibration}. 
While the introduction of $\phi$ and $\psi$ could seem arbitrary, it can be
thought in a more intrinsic manner by considering the embedding
$\phi(y) = \delta_y$ belonging to the Banach space ${\cal H}$ of signed measured,
$g^{*}(x) = \rho\vert_{x}$, with the bracket operator, for $\mu\in{\cal H}$ and
$z\in{\cal Z}$, $\scap{z}{\mu} = \int_{\Y} \ell(z, y)\mu(\diff y)$, and the
distance between signed measures being
$d(\mu_{1}, \mu_{2}) = \sup_{z\in{\cal Z}} \scap{z}{\mu_{1} - \mu_{2}}$. 
Note that Lemma \ref{lem:calibration} is a pointwise result,
holding for any $x\in\X$, that is integrated over $\X$ afterwards.
Therefore, it is enough to consider $\X = \brace{x}$ and remove the dependency
in $\X$ to understand it.
The simplex $\prob{\Y}$ naturally splits into decision region $R_z$ for
$z\in{\cal Z}$ as illustrated on Figure~\ref{fig:separation}.
The main idea of Lemma \ref{lem:calibration} is that one does not have to
precisely estimate $g^*(x) = \rho\vert_x$ but only has to make sure that $g_n(x)$ falls in the
same region on Figure~\ref{fig:separation}.

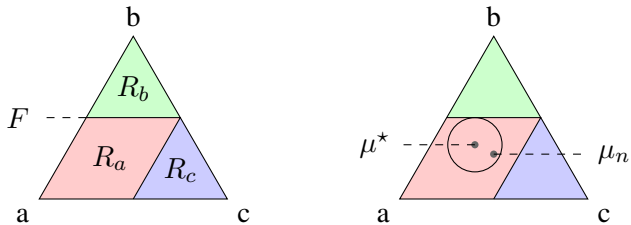
\begin{figure}[ht]
  \centering
  \begin{tikzpicture}[scale=2.5]
  \coordinate(a) at (0, 0);
  \coordinate(b) at ({1/2}, {sin(60)});
  \coordinate(c) at (1, 0);
  \coordinate(ha) at ({3/4}, {sin(60)/2});
  \coordinate(hb) at ({1/2}, 0);
  \coordinate(hc) at ({1/4}, {sin(60)/2});

  \fill[fill=red!20] (a) -- (hb) -- (ha) -- (hc) -- cycle;
  \fill[fill=green!20] (b) -- (ha) -- (hc) -- cycle;
  \fill[fill=blue!20] (c) -- (ha) -- (hb) -- cycle;

  \draw (a) node[anchor=north east]{a} -- (b) node[anchor=south]{b} --
  (c) node[anchor=north west]{c} -- cycle;
  \draw (hb) -- (ha) -- (hc);

  \draw[dashed] (hc) -- (0, {sin(60)/2}) node[left] {$F$};
  \node at ({3/8}, {sin(60)/4}) {$R_a$};
  \node at ({1/2}, {3*sin(60)/4 - 1/16}) {$R_b$};
  \node at ({3/4}, {sin(60)/4 - 1/16}) {$R_c$};
\end{tikzpicture}
\hspace{1cm}
\begin{tikzpicture}[scale=2.5]
  \coordinate(a) at (0, 0);
  \coordinate(b) at ({1/2}, {sin(60)});
  \coordinate(c) at (1, 0);
  \coordinate(ha) at ({3/4}, {sin(60)/2});
  \coordinate(hb) at ({1/2}, 0);
  \coordinate(hc) at ({1/4}, {sin(60)/2});

  \fill[fill=red!20] (a) -- (hb) -- (ha) -- (hc) -- cycle;
  \fill[fill=green!20] (b) -- (ha) -- (hc) -- cycle;
  \fill[fill=blue!20] (c) -- (ha) -- (hb) -- cycle;

  \draw (a) node[anchor=north east]{a} -- (b) node[anchor=south]{b} --
  (c) node[anchor=north west]{c} -- cycle;
  \draw (hb) -- (ha) -- (hc);
  \fill[fill=black!60] (.4, {sin(60)/3}) circle (.05em);
  \fill[fill=black!60] (.5, {sin(60)/3 -.05}) circle (.05em);
  \draw (.4, {sin(60)/3}) circle ({sin(60)/6});
  \draw[dashed] (.4, {sin(60)/3}) -- (0, {sin(60)/3}) node[left] {$\mu^\star$};
  \draw[dashed] (.5, {sin(60)/3-.05}) -- (1, {sin(60)/3-.05}) node[right] {$\mu_n$};
\end{tikzpicture}
  \caption{Illustration of Lemma \ref{lem:calibration}. Simplex $\prob{\Y}$, for
    $\Y = {\cal Z} = \brace{a, b, c}$ and $\ell$ a symmetric loss defined as $\ell(a, b) =
    \ell(a, c) = 1$ and $\ell(b, c) = 2$, while $\ell(z, z) = 0$.
    This leads to the decision regions $R_z$ represented in colors.
    Given $x\in\X$, if $g^*(x)$ corresponds to a distribution $\mu^* :=
    \rho\vert_{x}$ falling in $R_a$, and if $g_n(x)$ represented by $\hat\mu$
    falls closer to $\mu^*$ than the distance between $\mu^*$ and the decision
    frontier $F$ (represented by a circle on the right figure), then $\hat\mu$
    is also in $R_a$, and therefore $f^*(x)=f_n(x)=a$.
  }
  \label{fig:separation}
  \vspace*{-.5cm}
\end{figure}

\section{Rate acceleration under margin condition}
\label{sec:margin}

In this section, we introduce a condition that $g^{*}$ is not too often close to
the decision frontier~$F$. It generalizes the so-called ``Tsybakov margin
condition'' known for classification.
Under this condition, we proves rates that generalize the results of \cite{Audibert2007}
from binary classification to generic structured prediction problems,
which opens the way to ``super fast'' rates in structured prediction.

\subsection{No density separation}
To get fast convergence rates, one has to make assumptions on the problem. A
classical assumption is that $g^{*}$ is smooth enough in order to get concentration
bounds similar to Assumption \ref{ass:concentration} when considering a specific class of
estimates $g_{n}$.
In our decoding setting (Lemma \ref{lem:cal}),   learning is made easy when it is
easy to estimate in which region $R_z$ the optimal  $g^{*}$ will fall in. This is in particular
the case, when there is a margin $t_{0}>0$, for which, for no point $x \in
\supp\rho_{\X}$, $g^{*}(x)$ falls at distance $t_{0}$ of the decision frontier $F$,
motivating the following definition.

\begin{assumption}[No-density separation]\label{ass:no-density}
  A surrogate solution $g^{*}$ will be said to satisfy the \emph{no-density separation},
  if there exists a $t_{0} > 0$, such that
  \begin{equation}
    \label{eq:no-density}
    \Pbb_{X}(d(g^{*}(X), F) < t_{0}) = 0.
  \end{equation}
  This condition is alternatively called the \emph{hard margin condition}, or sometimes ``Massart's noise condition'' for binary classification~\citep{massart2006}.
\end{assumption}

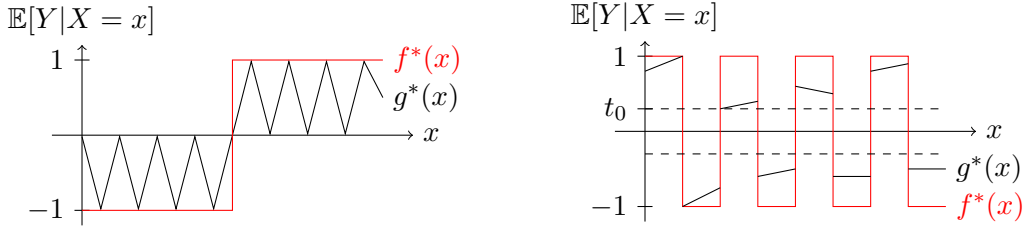
\begin{figure}[h]
  \vspace*{-.5cm}
  \centering
  \begin{tikzpicture}[xscale=2]
  \def\a{.02}
  \draw[->] (-1.2, 0) -- (1.2, 0) node[right] {$x$};
  \draw[->] (-1, -1.2) -- (-1, 1.2) node[above] {$\E[Y\vert X=x]$};
  \draw[-] (-.95, -1) -- (-1.05, -1) node[left] {$-1$};
  \draw[-] (-.95, 1) -- (-1.05, 1) node[left] {$1$};
  \draw[-] (0, 0)--(.125, 1-\a)--(.25, \a)--(.375, 1-\a)--(.5, \a)--(.625, 1-\a)--(.75, \a)--(.875, 1-\a)--(1, .5) node[right] {$g^*(x)$};
  \draw[-] (-1, -\a)--(-.875, -1+\a)--(-.75, -\a)--(-.625, -1+\a)--(-.5, -\a)--(-.375, -1+\a)--(-.25, -\a)--(-.125, -1+\a)--(0, 0);
  \draw[-,red] (-1, -1) -- (0, -1) -- (0, 1) -- (1, 1) node[right] {$f^*(x)$};
\end{tikzpicture}
\hspace{1cm}
\begin{tikzpicture}[xscale=2]
  \draw[->] (-1.2, 0) -- (1.2, 0) node[right] {$x$};
  \draw[->] (-1, -1.2) -- (-1, 1.2) node[above] {$\E[Y\vert X=x]$};
  \draw[-] (-.95, -1) -- (-1.05, -1) node[left] {$-1$};
  \draw[-] (-.95, 1) -- (-1.05, 1) node[left] {$1$};
  \draw[-,red] (-1,1)--(-.75, 1)--(-.75, -1)--(-.5, -1)--(-.5, 1)--(-.25, 1)--(-.25, -1)--(0, -1)--
  (0,1)--(.25, 1)--(.25, -1)--(.5, -1)--(.5, 1)--(.75, 1)--(.75, -1)--(1, -1) node[right] {$f^*(x)$};
  \draw[-] (-1,.8)--(-.75,1);
  \draw[-] (-.75,-1)--(-.5,-.75);
  \draw[-] (-.5,.3)--(-.25,.4);
  \draw[-] (-.25,-.6)--(0,-.5);
  \draw[-] (0,.6)--(.25,.5);
  \draw[-] (.25,-.6)--(.5,-.6);
  \draw[-] (.5,.8)--(.75,.9);
  \draw[-] (.75,-.5)--(1,-.5) node[right] {$g^*(x)$};
  \draw[dashed] (-1,.3)--(1,.3);
  \draw[dashed] (-1,-.3)--(1,-.3);
  \draw[-] (-.95,.3) -- (-1.05,.3) node[left] {$t_0$};
\end{tikzpicture}
  \caption{Illustration of Remark \ref{rmk:separation}.
    We represent two instances of binary classification (see Examples
    \ref{ex:binary} and \ref{ex:binary-2}).
    On the left example, when $\rho_{\X}$ is such that there is no mass where
    the sign of $g^{*}$ changes, classes are separated in $\X$, yet the no-density
    separation is not verified. On the right, classes are not separated in $\X$,
    but the problem satisfies the no-density separation as there is no $x$ such
    that $d(g^{*}(x), F) = \abs{g^{*}(x)} < t_{0}$. Note that when $\rho_{\X}$ is
    uniform, the left problem satisfies a milder separation condition, introduced
    thereafter and called the $1$-low-density separation.}
  \label{fig:no-density}
\end{figure}

\begin{remark}[Separation in $\Y$ and separation in $\X$]\label{rmk:separation}
  It is important to realize that Eq.~\eqref{eq:no-density} is a condition of separation in
  $\prob{\Y}$ that should hold for all $x\in\X$, but it does not state any
  separation between classes in $\X$ for $f^{*}:\X\to{\cal Z}$. 
  To visualize it, consider the classification problem where $\X = [-1,1]$,
  $\Y = {\cal Z} = \brace{-1, 1}$ and $\ell(z, y) = \ind{z\neq y}$.  
  \begin{itemize}
  \item 
    A situation where $\rho_{\X}$ is uniform on $\X$ and
    $\E[Y\vert X=x] = 2\cdot\ind{x\in p\N + \brace{a\midvert \abs{a} < p/4}} -
    1$, for $p = 1/50$, satisfies separation in
    $\prob{\Y}$ (Eq. \eqref{eq:no-density}), but classes are not separated in $\X$.
  \item A situation where $\rho$ is uniform on $[-1, -.5] \cup [.5, 1]$, with
    $\E[Y\vert X=x] = \sign(x) (1 - \abs{x})^{p}$, for $p>0$, satisfies a separation of classes
    in $\X$ but does not satisfy Eq. \eqref{eq:no-density}.
  \end{itemize}
  Note that continuity of $g^{*}$ and the no-density separation in Eq.
  \eqref{eq:no-density} imply separation of classes in $\X$.
  Note also that to get concentration inequality such as
  Eq.~\eqref{eq:concentration}, one usually supposes that $g^{*}$ is smooth.
  We refer the curious reader to Section 2.4 in \citet{Steinwart2007} for separation in $\X$.
\end{remark}
The introduction of Assumption \ref{ass:no-density} is motivated by the
following result. 

\begin{theorem}[Rates under no-density separation]\label{thm:no-density}
  When $\ell$ is bounded by $\ell_{\infty}$ ({\em i.e.}, $\ell(z, y) \leq \ell_{\infty}$ for any $(z, y) \in {\cal Z} \times \Y$) and satisfies Assumption \ref{ass:loss}, and $\mathcal{Z}$ is finite,
  under the no-density separation Assumption \ref{ass:no-density}, and the concentration Assumption \ref{ass:concentration},
  the excess risk is controlled
  \[
    \E_{{\cal D}_{n}}{\cal R}(f_{n}) - {\cal R}(f^{*}) \leq \ell_{\infty} \exp\paren{-\frac{L_n t_0^2}{1 + M_n t_0}}.
  \]
\end{theorem}
\begin{proof}
  Because we make a mistake only when $d(g^{*}(x), F) \geq \norm{g_{n}(x) -
    g^{*}(x)}$, we make no mistake when $\norm{g_{n}(x) - g^{*}(x)} < t_{0}$; otherwise
  we can consider the worse error we are going to pay, that is~$\ell_{\infty}$, leading to
  \[
    {\cal R}(f_{n}) - {\cal R}(f^{*}) \leq \ell_{\infty} \Pbb_{X}(\norm{g_{n}(x) - g^{*}(x)} > t_{0}).
  \]
  Taking the expectation with respect to ${\cal D}_{n}$ and using the fact that
  \(\E_A\Pbb_B(Z) = \E_A\E_B[\ind{Z}] =
  \E_B\E_A[\ind{Z}] = \E_B\Pbb_A(Z)\), and plug-in the concentration inequality Eq.~\eqref{eq:concentration}, we get the result.
\end{proof}

\begin{example}[Image classification]
  In image classification, one can arguably assume that the class of an
  image is a deterministic function of this image. With the 0-1 loss, it implies
  that the image classification problem verifies the no-density separation.
  The same holds for any discrete problem where the label is a deterministic
  function of the input.
  Based on Theorem \ref{thm:no-density} and Eq. \eqref{eq:concentration} in
  which $M$ is generally a constant when $L$ is proportional to the number of
  data, it is reasonable to ask for exponential convergences rates on such problems.
\end{example}

\subsection{Low density separation}
While we presented the no-density separation first for readability, it is a
strong assumption. Recall our example, Remark \ref{rmk:separation}, with $\E[Y\vert X=x] = \sign(x)(1 - \abs{x})^{p}$,
only around $x=1$ and $x=-1$ is $d(g^{*}(x), F)$ not bounded away from zero.
While the neighborhood of those points should be studied carefully, the error on all other
points $x \in [-1+t, 1-t]$ can be controlled with exponential rates.
The low-density separation, also known as the Tsybakov margin condition in
binary classification, will allow a refined control to get fast rates in such a setting.
\begin{assumption}[Low-density separation]\label{ass:low-density}
  A surrogate solution $g^{*}$ is said to satisfy the \emph{low-density separation},
  if there exists   $c_{\alpha} > 0$, and $\alpha > 0$, such that for any $t > 0$
  \begin{equation}
    \label{eq:low-density}
    \Pbb_{X}(d(g^{*}(X), F) < t) \leq c_{\alpha} t^{\alpha}.
  \end{equation}
  This condition is alternatively called the \emph{margin condition}.
\end{assumption}
The low-density separation spans all situations from the hard
margin condition, that can be seen as $\alpha = +\infty$, to situations without
any margin assumption corresponding to $\alpha = 0$.
The coefficient~$\alpha$ is an intrinsic measure of the easiness of finding $f^{*}$ in
the problem $(\ell, \rho)$.
For example, the setting described in the last paragraph corresponds to the case $\alpha = 1/p$.
We discuss the equivalence of Assumption \ref{ass:low-density} to definitions
appearing in the literature in Remark \ref{rmk:low-density}.

\begin{theorem}[Optimal rates under low density separation]\label{thm:low-density}
  Under refined calibration in Eq.~\eqref{eq:calibration}, concentration (Assumption \ref{ass:concentration}), and low-density separation (Assumption \ref{ass:low-density}),
  the risk is controlled as
  \[
    \E_{{\cal D}_{n}} {\cal R}(f_{n}) - {\cal R}(f^{*}) \leq
    2c_{\psi} c_{\alpha} c \paren{M_n^{\alpha + 1}L_n^{-(\alpha + 1)} + L_n^{-\frac{\alpha + 1}{2}}},
  \]
  for $c$ a constant that only depends on $\alpha$, that can be expressed
  through the Gamma function evaluated in quantity depending on $\alpha$,
  meaning that when $\alpha$ is big, $c$ behaves like $\alpha^{\alpha}$.
  Note that it is not possible to derive a better bound only given Eqs.
  \eqref{eq:concentration}, \eqref{eq:calibration} and \eqref{eq:low-density}. 
\end{theorem}
\begin{proof}[Sketch for Theorem \ref{thm:low-density}, details in Appendix \ref{proof:low-density}]
  Based on the refined calibration inequality in  Eq.~\eqref{eq:calibration}, and
  using that $\E[X] = \int_{0}^{\infty} \Pbb(X > t)\diff t$, it is possible to show
  that the expectation of the excess risk behave like
  \[
    \int_{0}^{\infty} \Pbb_{X}(d(g^{*}(x), F) < t) \sup_{x} \Pbb_{{\cal D}_{n}}(\norm{g_{n}(x)
      - g^{*}(x)} > t)\diff t.
  \]
  Based on Assumptions \ref{ass:concentration} and \ref{ass:low-density}, the
  integrand behaves like $t^{\alpha} \exp(- L_n t^{2} /(1+M_n t))$. A change of variable
  and the study of the Gamma function leads to the result. We provide all the
  details in Appendix \ref{proof:low-density}.
  Note that while we stated Theorem \ref{thm:low-density} under
  an exponential inequality of Bernstein type (Assumption
  \ref{ass:concentration}),
  similar theorems can be derived for any type of exponential concentration
  inequality, as stated in Lemma \ref{lem:ref-low-density} in Appendix
  \ref{app:ref-low-density}.
\end{proof}

Theorem \ref{thm:low-density} is to put in perspective with the work of \citet{Nowak2019} which considers the same setup as ours, yet only succeeds to derive acceleration by a power $2(\alpha+1) / (\alpha + 2)$, while we got an acceleration by a power $(\alpha + 1) / 2$ as already mentioned in the related work section. This gain will appear more clearly in Theorem \ref{thm:krr-low-density}.

\begin{remark}[Independence to the decomposition of $\ell$]\label{rmk:low-density}
  While we have stated results based on the quantity $d(g^{*}(x), F)$,
  generalization of the Tsybakov margin condition has also been expressed
  through the quantity 
  $\inf_{z\neq z^{*}} \E_{Y\sim \rho\vert_{x}}\ell(z, Y) - \E_{Y\sim \rho\vert_{x}}\ell(z^{*}, Y)$
  instead of $d(g^{*}(x), F)$ \citep{Nowak2019}.
  We show in Appendix
  \ref{proof:def-low-density} that the two definitions of the
  margin condition are equivalent.
\end{remark}

\begin{remark}[Scope of our work]\label{rmk:l2-pointwise}
  Our work relies on pointwise exponential concentration inequalities (Assumption
  \ref{ass:concentration}) which are specially designed to work well with the
  Tsybakov margin condition. 
  It is natural for localized averaging method such as nearest neighbors, or for
  surrogate methods leading to $L^\infty$ concentration.
  For surrogate methods leading to concentration of other quantities, it is
  possible to use similar tricks under different ``margin'' conditions
  ({\em e.g.} \citet{Steinwart2007} for a margin condition designed for the
  Hinge loss).
  Note that $L^2$ concentration on $g_n$ towards $g^*$ (such as the one derived
  by \citet{MarteauFerey2019} for logistic regression) could also be turned
  into fast convergence of $f_n$ towards $f^*$, since, in essence, for
  points $x\in\X$ where $\rho(\diff x)$ is high, the quantity $g^*(x) - g_n(x)$
  will have a non-negligible contribution to $\norm{g^* - g_n}_{L^2}$ --
  allowing to cast concentration in  $L^2$ to concentration pointwise in $x$ --
  and for points $x\in\X$ where $\rho(\diff x)$ is negligible, it is acceptable
  to pay the worst error, since it will have a small contribution on the excess
  of risk.
  Finally, note that it is also possible to let the right hand-side term in Eq.
  \eqref{eq:concentration} depends on $x$, and to modify Theorem
  \ref{thm:no-density} with $L = \E[L(X)]$.
\end{remark}

\subsection{The importance of constants}

In this subsection, we discuss on the importance of constants when providing
learning rates.
Assumption \ref{ass:no-density} corresponds to asking for $g^*(x)$ never to
enter a neighborhood of $F$ defined through~$t_0$.
Similarly, when $\X$ is parameterized such that $\rho_{\X}$ is uniform, the
parameter $\alpha$ in Assumption~\ref{ass:low-density} corresponds to the speed
at which $g^*(x)$ ``get through'' the decision frontier $F$. In order to have a higher $\alpha$
and optimize the dependency in $n$ in the bound of Theorem
\ref{thm:low-density}, it is natural to think of infinitesimal
perturbations of $g^*$ to make it cross the boundary orthogonally (or even jump
over it and satisfy the no-density seperation).
To give a precise example, in binary classification, let us artificially add
smoothness to the function $g^*(x) = x^q$ when approaching zero. 
Consider $g^*:[0, 1]\to[-1,1], x\to c^{q-p}x^p\ind{x<c} + x^q\ind{x\geq c}$, and
$x$ uniform, and $p < q$.
In this setting, $\alpha$ can be taken anywhere in $[0, p^{-1})$.
Naively, we could ask for the biggest possible $\alpha$  in order to have the
best dependency in $n$ in the learning rates given by Theorem
\ref{thm:low-density}.
While this approach will higher $\alpha$, it will also higher $c_\alpha$,
compensating the gain one could expect from such a strategy.
Indeed, for $\alpha \in [0, p^{-1}]$, at best, we can take
$c_\alpha = \ind{\alpha < q^{-1}} + c^{1-q\alpha}\ind{\alpha \geq q^{-1}}$. 
This shows the importance to optimize both $\alpha$ and $ c_\alpha c$ to minimize
the lower bound appearing in Theorem \ref{thm:low-density} when given a fixed
number of sample $n$.

In a word, while we only give results that are optimized in $n$, when $n$ is
fixed, better bounds could be given by optimizing parameters and constants
simultaneously. For example, when $\X = \R^d$ and $g^*$ belongs to the Sobolev
space $H^m$ for all $m \in [0, m_*]$, and satisfies Assumption \ref{ass:low-density} for all
$\alpha \in [0, \alpha_*]$, we expect the best bound, that could be derived
from our proof technique, to be of form
\[
  \E_{{\cal D}_n}{\cal R}(f_n) - {\cal R}(f^*) \leq
  \min_{m \leq m_*, \alpha < \alpha_*} \alpha^\alpha c_\alpha c_\psi\norm{g^*}_{H^m}
  n^{-\frac{m(\alpha + 1)}{2 d}}.
\]
Yet, for simplicity, we will express those bounds as $b n^{-\frac{m_*(\alpha_* + 1)}{2d}}$,
for $b$ a big constant.

\section{Application to nearest neighbors}
\label{sec:nn}

In this section, we consider the Bayes approximate risk estimator proposed by
\citet{Stone1977}, with weights given by nearest neighbors \citep{Cover1957}.
We prove, under regularity assumptions, concentration inequalities similar to
Eq. \eqref{eq:concentration},
which allow us to derive exponential and polynomial rates.
Given samples $(X_{i},Y_{i}) \sim \rho^{\otimes n}$, $k \in \N$ and a metric $d$ on $\X$, the estimator is
\begin{equation}\label{eq:nn}
 \!\!\! g_{n}(x) = \sum_{i=1}^{n} \alpha_{i}(x) \phi(Y_{i}),\ \ 
  \text{with}\ \
  \alpha_{i}(x) = \left\{
    \begin{array}{cl}
      k^{-1} & \text{if}\quad \sum_{j=1}^{n} \ind{d(x, X_{j}) \leq d(x, X_{i})} < k\\
      0 & \text{if}\quad \sum_{j=1}^{n} \ind{d(x, X_{j}) < d(x, X_{i})} \geq k\\
      (pk)^{-1} & \text{else, with } p = \sum_{j=1}^{n} \ind{d(x, X_{j}) = d(x, X_{i})}.
    \end{array}\right.
\end{equation}
To study the convergence of $g_{n}$, we introduce the noise free estimator
$g_{n}^{*} = \sum_{i=1}^{n} \alpha_{i}(x)g^{*}(X_{i})$.
This will allow to separate the error due to the randomness of the labels $Y_{i} \sim
\rho\vert_{X_{i}}$, and the error due to the difference between
$g^{*}(x)$ and the averaging of $g^{*}$ on the neighbors of $x$ defining $g_{n}$. 
To control the fist error, we need a bounded moment condition on $\phi(Y)$. We
reuse an assumption from \citet{Bernstein1924}, that is classic in
machine learning \citep[\emph{e.g.,}][]{Caponnetto2007,Lin2020}.
\begin{assumption}[Sub-exponential moment of $\rho\vert_{x}$]\label{ass:moment}
  Suppose that there exists $\sigma^{2}, M > 0$ such that for any
  $x\in\supp\rho_{\X}$, for any $m \geq 2$, we have
  \[
    \E_{Y\sim\rho\vert_{x}}\bracket{\norm{\phi(Y) - g^{*}(x)}^{m}}
    \leq \frac{1}{2} m! \sigma^{2} M^{m-2}.
  \]
\end{assumption}

\begin{example}[Moment bound on $\phi(Y)$]
  Assumption \ref{ass:moment} is a classical assumption that is notably satisfied
  when $\phi(Y)$ is bounded by $M$, with $\sigma^2$ its variance, or when $\paren{\phi(Y)\midvert X}$ is
  Gaussian with covariance bounded by a constant independent of $X$ \citep[see a proof of this standard result by][]{Fischer2020}.
\end{example}

To control the second error, we notice, for $x\in\supp\rho_{\X}$, that the quantity
$\norm{g^{*}(x) - g_{n}^{*}(x)}$ behaves like $\sup_{x'\in{\cal B}(x, r)} \norm{g^{*}(x)
  - g^{*}(x')}$, with $r$ such that 
$\rho_{\X}({\cal B}(x, r)) \approx k/n$, such a $r$ modeling the distance between $x$ and its $k$-th neighbor. This motivates the following assumption.

\begin{assumption}[Modified Lipschitz condition \citep{Chaudhuri2014}]\label{ass:lipschitz}
  $g^{*}$ is said to verify the $\beta$-Modified Lispchitz condition if there
  exists $c_{\beta} > 0$ such that for any $x, x' \in \supp\rho_{\X}$
  \[
    \norm{g^{*}(x) - g^{*}(x')} \leq c_{\beta} \rho_{\X}({\cal B}(x, d(x,x')))^{\beta},
  \] 
  where $d$ is the distance on $\X$, and ${\cal B}(x, t)\subset \X$ the ball of center $x$
  and radius $t$.
\end{assumption}

Typically the $\beta$ that appears in Assumption \ref{ass:lipschitz} is linked with
the dimension of a subset of $\X$ containing most the mass of $\rho_{\X}$ (see below). This
will slow the rates accordingly to this dimension parameter, a property referred to
as the curse of dimensionality.

\begin{example}[Classical assumptions]\label{ex:nn}
  When $\X = \R^{d}$, if $g$ is $\beta'$-H\"older continuous, and $\rho_{\X}$ is
  regular in the sense that, there exists a constant $c$ and $t^{*} > 0$ such that for
  $x\in\supp\rho_{\X}$ and any $t \in [0, t^{*}]$, 
  \(
    \rho_{\X}({\cal B}(x, t)) \geq c \lambda({\cal B}(x, t)),
  \)
  with $\lambda$ the Lebesgue measure on $\X$, then $g$ satisfies the modified
  Lipschitz condition with $\beta = \beta' / d$.
  The condition on $\rho_{\X}$ is usually split in a condition of 
  minimal mass of $\rho_{\X}$, and a condition of
  regular boundaries of $\supp\rho_{\X}$ \citep[\emph{e.g.,}][]{Audibert2007}. We
  provide more details in Appendix \ref{app:nn-assumptions}.
\end{example}
We now state convergence results, respectively proven in
Appendices \ref{proof:nn-concentration}, \ref{proof:nn-no-density} and \ref{proof:nn-low-density}, in which the constant values $b_1$ to $b_6$ appear explicitly.
Note that results provided by Lemma \ref{lem:nn-concentration} are already known in the literature \citep{Gyorfi2002}, while Theorems \ref{thm:nn-no-density} and 
\ref{thm:nn-low-density} were only known in binary classification, but we generalize them to any discrete structured prediction problem.
It should be noted that rates in Theorem \ref{thm:nn-low-density} match the minimax rates derived by \cite{Audibert2007} in the case of binary classification. 
\begin{lemma}[Nearest neighbors concentration]\label{lem:nn-concentration}
  Under Assumptions \ref{ass:moment} and \ref{ass:lipschitz}, there exist constants
  $b_{1}, b_{2}, b_{3} > 0$, such that for any $x\in\supp\rho_{\X}$ and any $t > 0$,
  \[
    \Pbb_{{\cal D}_{n}}\paren{\norm{g_{n}(x) - g_{n}^{*}(x)} > t} \leq 
    2\exp\paren{-\frac{b_{1}kt^{2}}{1 + b_{2}t}}. 
  \]
  And for $t > \paren{k/2n}^{\beta}$, when $\rho_{\X}$ is continuous\footnote{Note that this topological assumption ease derivations but is not fundamental for such non-asymptotic results.}
  \[
    \Pbb_{{\cal D}_{n}}\paren{\norm{g_{n}^{*}(x) - g^{*}(x)} > t} \leq 
    \exp\paren{-b_{3} nt^{\frac{1}{\beta}}}. 
  \]
\end{lemma}
\begin{theorem}[Nearest neighbors fast rates under no-density assumption]
  \label{thm:nn-no-density}
  When $\ell$ is bounded by $\ell_{\infty}$, satisfies Assumption \ref{ass:loss}, and $\mathcal{Z}$ is finite,
  under the no-density separation, Assumption \ref{ass:no-density}, and
  Assumptions \ref{ass:moment} and \ref{ass:lipschitz}, there exist  two
  constants $b_{4}, b_{5} > 0$ that do not depend on $n$ or $k$ such that for any
  $n \in \N^{*}$ and any $k$ such that $\paren{k/2n}^{\beta} < t_{0}$, we have
  \begin{equation}
    \E_{{\cal D}_{n}}{\cal R}(f_{n}) - {\cal R}(f^{*}) \leq
    2\ell_{\infty}\exp(-b_{4}k) + \ell_{\infty}\exp(-b_{5}n).
  \end{equation}
\end{theorem}

\begin{theorem}[Nearest neighbors fast rates under low-density assumption]
  \label{thm:nn-low-density}
  When $\ell$ satisfies Assumption \ref{ass:loss}, and ${\cal Z}$ is finite, under the
  low-density separation, Assumption \ref{ass:low-density}, and
  Assumptions \ref{ass:moment} and \ref{ass:lipschitz}, considering the scheme
  $k_{n} = \floor{k_{0}n^{\frac{2\beta}{2\beta + 1}}}$, for any $k_{0} > 0$,
  there exists a constant $b_{6} > 0$ that does not depend on $n$ such that for
  any $n \in \N^{*}$,
  \begin{equation}\label{eq:nn_rates}
    \E_{{\cal D}_{n}}{\cal R}(f_{n}) - {\cal R}(f^{*}) \leq
    b_{6} n^{-\frac{\beta(\alpha + 1)}{2\beta + 1}}.
  \end{equation}
\end{theorem}

\begin{remark}[Scope of our work]
  The same type of argument works for other local averaging methods, such as
  Nadaraya-Watson \citep{Nadaraya1964,Watson1964}, local polynomials
  \citep{Cleveland1979,Audibert2007} or decision trees \citep{Breiman1984}.
\end{remark}
\begin{figure}[ht]
  \vspace*{-.5cm}
  \centering
  \includegraphics{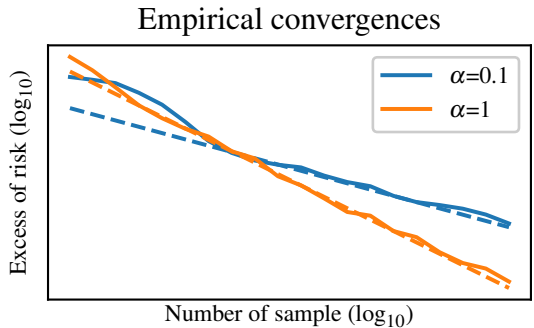}
  \caption{Empirical convergence rates. We consider binary classification, with
    $\X = [-1, 1]$, $g^{*}(x) = \sign(x)*\abs{x}^{\frac{1}{\alpha}}$, 
    for $\alpha \in \brace{.1, 1}$ and $\rho_{\X}$ uniform.
    We plot in solid the logarithm of the excess risk averaged over
    100 trials against the logarithm of the number of samples for $n \in [10,
    10^6]$, and plot in dashed the expected slope of those curves due to Theorem
    \ref{thm:nn-low-density} (\emph{i.e.}, we fit the constant $C$ in the rate $Cn^{-\gamma}$ with $\gamma$ obtained from the bound in Eq.~\eqref{eq:nn_rates}).}
  \label{fig:rates}
\end{figure}

\section{Application to reproducing kernel ridge regression}
\label{sec:rkhs}

In this section, we consider the kernel ridge regression estimate $g_{n}$ of $g^{*}$
first proposed by \cite{Ciliberto2016}, and we prove, under regularity assumptions,
uniform concentration inequalities similar to Eq.~\eqref{eq:concentration},
which allow us to derive super fast rates at the end of the section.
Given a symmetric, positive semi-definite kernel $k:\X\times\X\to\R_+$,
the kernel ridge regression estimation $g_{n}$ of $g^{*}$ is defined similarly to
Eq. \eqref{eq:nn} yet with weights $\alpha(x) \in \R^{n}$ defined as
\[
  \alpha(x) = (\hat K + \lambda)^{-1} \hat K_{x},\quad
  \hat K = \paren{\frac{1}{n} k(X_{i}, X_{j})}_{i,j\leq n} \in \R^{n\times n},\ 
  \hat K_{x} = \paren{\frac{1}{n} k(x, X_{i})}_{i\leq n} \in \R^{n}.
\]
To state regularity assumptions, we introduce a minimal setup linked to the
reproducing kernel $k$. To keep the exposition clear, we relegate
technicalities in Appendix \ref{app:rkhs}. We define the operator $K$ operating on
functions $f\in L^{2}(\X, {\cal H}, \rho_{\X})$ and $K_{\X}$ operating on $f\in L^{2}(\X, \R, \rho_{\X})$, both defined as
\[
  (Kf)(x') = \int_{\X} k(x', x)f(x)\diff\rho_{\X}(x). 
\]
Inheriting from the symmetry and positive
semi-definiteness of $k$, $K$ is self-adjoint with spectrum in $\R_+$.
To study the convergence of $g_{n}$ to $g^{*}$, it is useful to introduce the
approximate orthogonal projection on $\ima K^{\frac{1}{2}}$, defined for
$\lambda > 0$ as 
\[
  g_\lambda = K (K+\lambda)^{-1} g^{*}.
\]
We introduce three assumptions linked with the regularity of the problem, referred to as the capacity condition, interpolation inequality and source condition. Those are classical assumption to prove uniform rates of the kernel ridge regression estimates. They could be found, in particular, by \citet{Fischer2020} under the respective names of (EVD), (EMB) and (SRC), but also by \citet{PillaudVivien2018,Lin2020}. Our assumptions differ in that they are expressed for vector-valued functions, which usually generate compactness issues \citep{Caponnetto2007}. However, when ${\cal Z}$ is finite, ${\cal H}$ is finite dimensional, and $K$ can be shown to be a compact operator, thus allowing to consider fractional power without definition issues.

\begin{assumption}[Capacity condition]\label{ass:capacity}
  Suppose $\trace\paren{K_{\X}^{\sigma}} < +\infty$ for
  $\sigma \in [0, 1]$.
\end{assumption}
\begin{assumption}[Interpolation inequality]\label{ass:interpolation}
  Assume the existence of $p \in [0,\frac{1}{2}]$, $c_p > 0$ such that
  \[
    \forall\ g\in \paren{\ker K}^{\perp}, \qquad
    \norm{K^{p}g}_{L^{\infty}} \leq c_p\norm{g}_{L^{2}}.
  \]
\end{assumption}
\begin{assumption}[Source condition]\label{ass:source} Suppose 
  $g^{*} \in \ima K^{q}$ for $q \in (p, 1]$.
\end{assumption}

When $q=\sfrac{1}{2}$, the source condition is often
expressed as $g^{*}$ belonging to the reproducing kernel Hilbert space associated
to the kernel $k$.
Note that when $k$ is bounded, Assumptions~\ref{ass:capacity} and~\ref{ass:interpolation} hold with $\sigma=1$ and $p=\sfrac{1}{2}$.
In those assumptions, for $p$ and $\sigma$ the smaller, and for $q$ the bigger, the faster the convergence rates will be.

\begin{example}[Classical assumptions]
  For Assumption \ref{ass:interpolation} to hold, minimal mass
  and regular support of $\rho$, similarly to Example \ref{ex:nn}, are often
  assumed, as well as regularity of functions in $\ima K^{p}$, in coherence
  with Remark \ref{rmk:l2-pointwise}.
  For Assumption \ref{ass:source} to hold, it is classical to assume regularity
  of $g^{*}$, matching the regularity of function spaces
  derived from the kernel $k$.
  The value of $\sigma$ in Assumption \ref{ass:capacity} often comes has a bonus
  of regularity assumptions on $\rho$ and specificity of the RKHS implied by $k$.
  See Example 2 by \citet{PillaudVivien2018}
  and Section 4 by \citet{Fischer2020} as well as references therein for
  concrete examples.
\end{example}
We now state convergence results respectively proven in
Appendices \ref{proof:krr-1} and \ref{proof:krr-2}, \ref{proof:krr-no-density},
and \ref{proof:krr-low-density}.
Lemma \ref{lem:rkhs} is a generalization to vector-valued
functions of kernel ridge regression uniform convergence rates known for real-valued
function \citep[see][]{Fischer2020}.
Note that a similar result to Theorem \ref{thm:krr-no-density} was
provided for binary classification by \citet{Koltchinskii2005}, but we generalize exponential rates with kernel ridge regression to any discrete structured prediction problem.
Theorem \ref{thm:krr-low-density} is new, even in the context of binary classification. It states that, while, up to now, only rates in
$n^{-1/4}$ were known for $f_{n}$ \citep{Ciliberto2020}, one can indeed hope for
arbitrarily fast rates, depending on the hardness of the problem, read in the
value of $\alpha\in[0, \infty)$.

\begin{lemma}[Reproducing kernel concentration]\label{lem:rkhs}
  Under Assumptions \ref{ass:capacity}, \ref{ass:interpolation} and
  \ref{ass:source}, for any $\lambda > 0$, 
  \[
    \norm{g_\lambda - g^{*}}_{L^{\infty}} \leq b_{1} \lambda^{q-p}.
  \]
  With $b_{1} = c_p\norm{K^{-q}g^{*}}_{L^{2}}$.
  Moreover, when the kernel $k$ is bounded and under Assumption \ref{ass:moment}, there
  exists three constants $b_{2}, b_{3}, b_{4}, b_{5} > 0$ that does not depend nor on
  $\lambda$ nor on $n$ such that
  \[
    \Pbb\paren{\norm{g_{n} - g_\lambda}_{\infty} > t} \leq
    b_{2} \lambda^{-\sigma}\exp\paren{-b_{3}n\lambda^{2p}}
    + 4\exp\paren{-\frac{n\lambda^{2p + \sigma}t^{2}}
      {b_{4}+b_{5}t}}.
  \]
  As long as $b_{3}n \geq \lambda^{-p}$, and $\lambda \leq \min\paren{\norm{K}_{\op}, 1}$.
\end{lemma}

\begin{theorem}[Kernel ridge regression fast rates under no-density assumption]
  \label{thm:krr-no-density}
  When the loss $\ell$ is bounded, satisfies Assumption \ref{ass:loss} and ${\cal Z}$ is finite, 
  under the $t_{0}$-no-density separation condition, and Assumptions
  \ref{ass:moment}, when $k$ is bounded,
  if $\lambda_{n} = \lambda$, for any $\lambda > 0$ such that $\norm{g^{*} - g_\lambda}_{L^{\infty}} < t_{0}$,
  then there exist two constants $b_{6}, b_{7} > 0$ such that, for any $n\in\N^{*}$,
  \begin{equation}
    \E_{{\cal D}_{n}} {\cal R}(f_{n}) - {\cal R}(f^{*}) \leq b_{6} \exp(-b_{7} n),
  \end{equation}
  with $f_{n}$ given by the kernel ridge regression surrogate estimate.
\end{theorem}

\begin{theorem}[Kernel ridge regression fast rates under low-density assumption]
  \label{thm:krr-low-density}
  When $\ell$ satisfies Assumption \ref{ass:loss}, is bounded and ${\cal Z}$ is finite,
  under the $\alpha$-low-density separation condition, and Assumptions \ref{ass:moment},
  \ref{ass:capacity}, \ref{ass:interpolation} and \ref{ass:source}, if
  $\lambda_{n} = \lambda_{0} n^{-\frac{1}{2q + \sigma}}$, for any $\lambda_{0} > 0$,
  there exists $b_{8} > 0$, such that for any $n \in \N^{*}$,
  \begin{equation}
    \E_{{\cal D}_{n}} {\cal R}(f_{n}) - {\cal R}(f^{*}) \leq b_{8} 
    n^{-\frac{(q-p)(1+\alpha)}{2q + \sigma}}.
  \end{equation}
\end{theorem}

\section{Conclusion}

In this paper, we have shown how, for discrete problems, to leverage
exponential concentration inequalities derived on continuous surrogate
problems, in order to derive faster rates than rates directly obtained through
calibration inequalities.
Those rates are arbitrarily fast, depending on a parameter characterizing the
hardness of the discrete problem.
We have shown how this method directly applies to local averaging methods and to
kernel ridge regression, which allowed us to derive ``super fast'' rates for any
discrete structured prediction problem.

This opens the way to several follow-up, such as
\begin{itemize}
  \item Applicative follow-up, consisting of tackling concrete problem
    instances, such as predicting properties of DNA-sequence \citep{Jaakkola2000}, 
    \emph{e.g.}, gene mutations responsible for diseases, 
    with well-designed kernels on DNA in order
    to higher the exponent appearing in Theorem \ref{thm:krr-low-density}.
  \item Computational follow-up, pushing our analysis further to understand how
    to design better algorithms on discrete problems. 
    For example, by adding a
    regularization pushing $g_{n}$ away from the decision frontier $F$, and
    adding a term in $\ind{\norm{g_{n}(x) - g^{*}(x)} > d(g_{n}(x), F)}$ in 
    Eq. \eqref{eq:calibration} for the analysis. 
  \item Theoretical follow-up, to widen our analysis to other types of smooth
    surrogates, and to parametric methods, such as deep learning models, assuming
    that functions are parameterized by a parameter $\theta$, that some analysis
    gives concentration on $\theta_{n} - \theta^{*}$ similar to Eq.
    \eqref{eq:concentration} and that calibration inequalities relate the error
    on $\theta$ with the error between $f_{n}=f_{\theta_{n}}$ and $f^{*} = f_{\theta^{*}}$.
\end{itemize}


\acks{
We would like to thanks Alex Nowak for useful discussions, 
as well as anonymous referees for helpful comments.
This work was funded in part by the French government
under management of Agence Nationale de la Recherche as part of the
``Investissements d'avenir'' program, reference ANR-19-P3IA-0001 (PRAIRIE 3IA Institute). We also acknowledge support from the European Research Council (grants SEQUOIA 724063 and REAL 947908).
}

\bibliography{main}

\clearpage
\appendix

\section{Fast rates}
\label{sec:proof}

In the following, we consider $\X$ and $\Y$ to be Polish spaces, {\em i.e.},
separable completely metrizable topological spaces, in order to
define the distribution $\rho$.
We also consider $\cal Z$ endowed with a topology that makes it compact, and that
makes $z\to\E_{Y\sim\mu}\ell(z, Y)$ continuous for any $\mu\in\prob{\Y}$, in
order to have minimizer well defined.
For a Polish space ${\cal A}$, we denote by $\prob{\cal A}$ the simplex formed
by the set of Borel probability measures on this space.
For $\rho\in\prob{\X\times\Y}$, we denote by $\rho\vert_{x}$ the conditional
distribution of $Y$ given $x$, and by $\rho_{\X}$ the marginal distribution over
$\X$.
We suppose ${\cal H}$ separable Hilbert and that the mapping $\phi$ is measurable in order to
define the pushforward measure $\phi_*\rho\vert_{x}$.
We assume that, for $\rho_{\X}$-almost every $x$, $(\phi(Y)\vert X=x)$ has a
second moment, in order to consider the conditional mean $g^*(x)$ as 
the solution of the well defined problem consisting of minimizing
$\norm{\xi - \phi(Y)}^2$ for $\xi \in {\cal H}$.
We consider $\psi$ to be continuous, in order to have the decoding
problem well posed.

\subsection{Proof of Lemma \ref{lem:cal}}
\label{proof:cal}

With the notation of Lemma \ref{lem:cal},
for $x\in\supp\rho_{\X}$
\begin{align*}
  \E_{Y\sim\rho\vert_{x}}&\bracket{\ell(f_n(x), Y) - \ell(f^*(x), Y)}
  = \scap{\psi(f_n(x)) - \psi(f^*(x))}{g^*(x)}_{\cal H}
  \\& = \scap{\psi(f_n(x))}{g_n(x)}
  + \scap{\psi(f_n(x))}{g^*(x) - g_n(x)}
  - \scap{\psi(f^*(x))}{g^*(x)}
  \\& \leq \scap{\psi(f^*(x))}{g_n(x)}
  + \scap{\psi(f_n(x))}{g^*(x) - g_n(x)}
  - \scap{\psi(f^*(x))}{g^*(x)}
  \\& = \scap{\psi(f_n(x)) - \psi(f^*(x))}{g^*(x) - g_n(x)}
  \\& \leq \norm{\psi(f_n(x)) - \psi(f^*(x))}_{\cal H}\norm{g^*(x) - g_n(x)}_{\cal H}
  \\& \leq 2c_{\psi}\norm{g^*(x) - g_n(x)}_{\cal H},
\end{align*}
where the inequality $\scap{\psi(f_n(x))}{g_n(x)} \leq \scap{\psi(f^*(x))}{g_n(x)}$ is due to the fact that $f_n(x)$ minimizes the functional $z\to \scap{\psi(z)}{g_n(x)}$.
Integrating over $\X$ leads to the results in Lemma \ref{lem:cal}.

\subsection{Proof of Lemma \ref{lem:calibration}}
\label{proof:calibration}

The first part of the lemma is a geometrical result stating that to go from two
elements $\xi_1$ and $\xi_2$ in $\prob{\phi(\Y)}$, leading to two different
decoding, one has to pass by a point $\xi_{1/2} \in F$, where there is at least two  possible decodings.
Let make it clearer.
Consider $x\in\supp\rho_{\X}$ and suppose that $f_n(x) \neq f^*(x)$, define the
path
\myfunction{\zeta}{[0,1]}{\prob{\phi(\Y)}}{\lambda}{\lambda g_n(x) +
  (1-\lambda)g^*(x).}
Consider $d:\prob{\phi(\Y)}\to{\cal Z}$ the decoding function used to retrieve $f^*$
and $f_n$, from $g^*$ and $g_n$, satisfying $d(\xi) \in \argmin_{z\in{\cal Z}} \scap{\psi(z)}{\xi}$.
Consider the path $d\circ\zeta:[0, 1] \to {\cal Z}$, it goes from $\zeta(0) = f^*(x)$
to $\zeta(1) = f_n(x)$. Consider $\lambda_{\infty}$ the supremum of
$(d\circ\zeta)^{-1}(f^*(x))$. We will show that $\zeta(\lambda_{\infty}) \in F$, this
will lead to
\[
  \norm{g_n(x) - g^*(x)} = \norm{g_n(x) - \zeta(\lambda_{\infty})}
  + \norm{\zeta(\lambda_{\infty}) - g^*(x)}
  \geq \norm{\zeta(\lambda_{\infty}) - g^*(x)}
  \geq d(g^*(x), F),
\]
and to Lemma \ref{lem:calibration} by contraposition.

To show that $\zeta(\lambda_{\infty}) \in F$, we will show that
\( f^*(x) \in \argmin_{z} \scap{\psi(z)}{\zeta(\lambda_{\infty})} \not\subset\brace{f^*(x)}. \)
By definition of the supremum,
there exists a sequence $(\lambda_p)_{p\in\N}$ converging to $\lambda_{\infty}$ such that
\[
  f^*(x) \in 
  \argmin_{z} \scap{\psi(z)}{\lambda_p g_n(x) + (1-\lambda_p)g^*(x)},
\]
meaning that for all $z\neq f^*(x)$
\[
  \scap{\psi(f^*(x))}{\lambda_p g_n(x) + (1-\lambda_p)g^*(x)}
  \leq \scap{\psi(z)}{\lambda_p g_n(x) + (1-\lambda_p)g^*(x)}.
\]
By continuity of the scalar product, it means that it holds for $p=\infty$, which
means $f^*(x) \in \argmin_z \scap{\psi(z)}{\zeta(\lambda_{\infty})}$.
Now, suppose that $\argmin_z\scap{\psi(z)}{\zeta(\lambda_{\infty})} =
\brace{f^*(x)}$, this means that for all $z\neq f^*(x)$,
\[
  \scap{\psi(f^*(x))}{\lambda_{\infty} g_n(x) + (1-\lambda_{\infty})g^*(x)}
  < \scap{\psi(z)}{\lambda_{\infty} g_n(x) + (1-\lambda_{\infty})g^*(x)}.
\]
By continuity of this function accordingly to $\lambda$, this means that this
still holds for $\lambda_{\infty} + \epsilon_z$ for $\epsilon_z > 0$.
Taking $\epsilon = \inf_{z\in{\cal Z}} \epsilon_z$, it means that
$\lambda_{\infty} + \epsilon \in (d\circ\zeta)^{-1}(f^*(x))$.
When ${\cal Z}$ is finite, $\epsilon > 0$, which contradict the definition of $\lambda_{\infty}$. Therefore
$\zeta(\lambda_{\infty}) \in F$.

The second part of Lemma \ref{lem:calibration} follows from derivations in
Appendix \ref{proof:cal}.

\begin{remark}[Extension to discrete cases]\label{rmk:extension-1}
  Note that the same argument can be generalized to discrete problems -- which
  could be defined as ${\cal Z}$ endowed with a topology that makes
  $z\to\E_{Y\sim\mu}[\ell(z, Y)]$ continuous with respect to $z$, and ${\cal
    Z}\backslash\brace{z}$ locally compact for any $z\in\cal Z$ -- that are not
  degenerate, in the sense that $\rho_{\X}$ almost all $x\in\X$, there exists $t > 0$ such that the 
  cardinality of the set defined as $\brace{z\midvert\E_{Y\sim\rho\vert_{x}}[\ell(z, Y)] -
    \inf_{z'\in\Y}\E_{Y\sim\rho\vert_{x}}[\ell(z', Y)] < t}$ 
  if finite. This holds for classification with infinite countable classes, but
  it does not for regression on the set of rational numbers.
\end{remark}

\begin{remark}[Extension to general cases]\label{rmk:extension-2}
  To remove the condition ${\cal Z}$ finite,
  one can change the definition of $d(g^*(x), F)$ to
  \(
  \inf_{\xi\in{\cal H}; \brace{f^*(x)} \neq \argmin\scap{\psi(z)}{\xi}} \norm{\xi - g^*(x)},
  \)
  in order to make Lemma \ref{lem:calibration} hold for any ${\cal Z}$.
\end{remark}

\subsection{Equivalence between generalizations of the Tsybakov margin condition}
\label{proof:def-low-density}

While we state the margin condition with $d(g^*(x), F)$, it could also be stated
with $d(g^*(x), F\cap\hull(\phi(\Y)))$ or with, which is the quantity considered by \citep{Nowak2019},
\[
  \gamma(x) = \inf_{z\neq z^*} \E_{Y\sim \rho\vert_{x}}\ell(z, Y) -
  \E_{Y\sim \rho\vert_{x}}\ell(z^*, Y)
  = \inf_{z\neq z^*}\scap{\psi(z) - \psi(z^*)}{g^*(x)}.
\]
Indeed, when $\cal Z$ is finite and $\ell$ is proper in the sense that
$\ell(\cdot, y) = \ell(\cdot, z)$ implies $z=y$, and that there is no $z$ that
minimizes a linear combination of $(\ell(\cdot, y))_{y\in\Y}$ without minimizing
a convex combination of the same family, we have the existence of two constants
such that 
\[
  c\gamma(x) \leq d(g^*(x), F\cap\hull(\phi(\Y))) \leq d(g^*(x), F) \leq c'\gamma(x).
\]

\subsubsection{Mildness of our condition}
Let $z'$ be the argmin defining $\gamma$, geometric properties of the scalar
product imply the existence of a $\xi \in (\phi(z') - \phi(z^*))^\perp$ such that
\[
  \scap{\phi(z') - \phi(z^*)}{g^*(x)}
  = \norm{\phi(z') - \phi(z^*)} \norm{g^*(x) - \xi}.
\]
Therefore
\[
  \scap{\phi(z') - \phi(z^*)}{g^*(x)}
  \geq \min_{y, y'} \norm{\phi(y) - \phi(y')}
  \norm{g^*(x) - \xi}.
\]
Note that, by definition of $\xi$, $\scap{\xi}{\phi(z')} = \scap{\xi}{\phi(z^*)}$.
If $\xi \in R_{z^*}$ then $\xi \in F$, otherwise $\xi \notin R_{z^*}$ and then,
there exists a point between $\xi$ and $g^*(x)$ that belongs to the decision
frontier (see Appendix \ref{proof:calibration} for a proof - for which we need some
regularity assumption such as $\cal Z$ finite). In every case,
\[
  \norm{g^*(x) - \xi} \geq d(g^*(x), F).
\]
This implies the existence of $c'$.

\subsubsection{Strength of our condition}
For any $g_n$ such that $f_n(x) = z$, we have
\begin{align*}
  \scap{\psi(z) - \psi(z^*)}{g^*(x)}
  &= \scap{\psi(z)}{g^*(x) - g_n(x)}
  + \scap{\psi(z)}{g_n(x)} - \scap{\psi(z^*)}{g^*(x)}
  \\&\leq \scap{\psi(z)}{g^*(x) - g_n(x)}
  + \scap{\psi(z^*)}{g_n(x)} - \scap{\psi(z^*)}{g^*(x)}
  \\&\leq 2 c_\psi\norm{g^*(x) - g_n(x)}.
\end{align*}
If we take the infimum on both sides we have
\[
  d(g^*(x), F) = \inf_{g_n(x)\notin R_{f^*(x)}} \norm{g_n(x) - g^*(x)}
  \geq \frac{1}{2c_\psi}\inf_{z\neq z^*} \scap{\psi(z) - \psi(z^*)}{g^*(x)},
\]
where the left equality is provided, when ${\cal Z}$ is finite, by a similar reasoning
to the one in Appendix \ref{proof:calibration}.
This implies the existence of $c$.
Note also that if the loss is proper in the sense that if $z$ minimizes
$\scap{\psi(z)}{\xi}$ for a $\xi\in{\cal H}$,
there exists a $\xi \in \hull{\phi(\Y)}$ such that $z$ minimizes 
$\scap{\psi(z)}{\xi}$, we can consider $g_n(x) \in \hull(\phi(\Y))$, and
therefore restrict $F$ to $F\cap\hull{\phi(\Y)}$.
Finally we have shown that, when $\cal Z$ finite and $\ell$ proper
\[
  c\gamma(x) \leq d(g^*(x), F\cap\hull(\phi(\Y))) \leq d(g^*(x), F) \leq c'\gamma(x).
\]
This explains why we would consider $\gamma(x)$, $d(g^*(x),
F\cap\hull(\phi(\Y)))$ or $d(g^*(x), F)$ to define the margin condition, it will
only change the value of constants in Assumptions \ref{ass:no-density} and \ref{ass:low-density}.

\subsection{Refinement of Theorem \ref{thm:no-density}}
\label{proof:no-density}

It is possible to refine Theorem \ref{thm:no-density} to remove the condition
that the loss $\ell$ is bounded.
In the following, we omit the dependency of $L_n$ and $M_n$ to $n$.

\begin{lemma}[Refinement of Theorem \ref{thm:no-density}]
  \label{lem:ref-no-density}
  Under refined calibration \eqref{eq:calibration}, concentration, Assumption
  \ref{ass:concentration}, and no-density separation, Assumption \ref{ass:no-density},
  the risk is controlled as
  \[
    \E_{{\cal D}_n} {\cal R}(f_n) - {\cal R}(f^*) \leq
    4c_\psi L^{-1/2} \exp\paren{-\frac{t_0^2 L}{2}}^{1/2}
    + 4c_\psi ML^{-1} \exp\paren{- \frac{t_0 L}{2M}}.
  \]
  Note that it is not possible to derive a better bound only given Eqs.
  \eqref{eq:concentration}, \eqref{eq:calibration} and \eqref{eq:no-density}.
  Yet when $\ell$ is bounded by $\ell_{\infty}$, we have 
  \[
    \E_{{\cal D}_n} {\cal R}(f_n) - {\cal R}(f^*) \leq \ell_{\infty}
    \exp\paren{-\frac{Lt_0^2}{1 + Mt_0}}.
  \]
\end{lemma}
\begin{proof}
Using the calibration inequality along with the no-density
separation one, we get
\begin{align*}
  {\cal R}(f_n) - {\cal R}(f^*)
  &\leq 2c_\psi \E_{X}\bracket{\ind{\norm{g_n(X) - g^*(X)} \geq t_0} \norm{g_n(X) - g^*(X)}} \\
    &= 2c_\psi \int_{t_0}^\infty \Pbb_{X}\paren{\norm{g_n(X) - g^*(X)} \geq t} \diff t.
\end{align*}
Taking the expectation over ${\cal D}_n$ and using concentration inequality we
have
\begin{align*}
  \E_{{\cal D}_n}{\cal R}(f_n) - {\cal R}(f^*)
  &\leq 2c_\psi \int_{t_0}^\infty \Pbb_{X, {\cal D}_n}\paren{\norm{g_n(X) - g^*(X)} \geq t} \diff t \\
    &\leq 2c_\psi \int_{t_0}^\infty \exp\paren{-\frac{Lt^2}{1+Mt}} \diff t.
\end{align*}
We only need to study the integral
\(
  \int_{t_0}^\infty \exp\paren{-\frac{Lt^2}{1+Mt}} \diff t.
\)
We first clean the dependency on $t$ insider the exponential using that
\[
  \frac{1}{2}\paren{\exp(-Lt^2) + \exp\paren{-\frac{Lt}{M}}}
  \leq \exp\paren{-\frac{Lt^2}{1 + Mt}}
  \leq \exp\paren{-\frac{Lt^2}{2}} + \exp\paren{-\frac{Lt}{2M}}.
\]
We are left with the study of
\(
  \int_{t_0}^\infty \exp(-At^p)\diff t,
\)
for $p \in \brace{1, 2}$ and $A > 0$. The case $p=1$, directly leads to
\(
  A^{-1}\exp(-At_0),
\)
explaining the part in $L / M$.
The case $p=2$ is similar to the Gaussian integral, and can be handle with the
following tricks
\begin{align*}
  \int_{t_0}^\infty \exp(-At^2)\diff t
  &= \frac{1}{2} \int_{(-\infty, -t_0]\cup[t_0, \infty)} \exp(-At^2)\diff t 
  \\&= \frac{1}{2}\paren{\int_{((-\infty, -t_0]\cup[t_0, \infty))^2} \exp(-A\norm{x}^2))
    \diff x}^{1/2}.
\end{align*}
This last integral corresponds to integrate the function $x\to\exp(-A\norm{x}^2)$ for
$x\in\R^2$ on the domain $((-\infty, -t_0]\cup[t_0, \infty))^2$.
This function being positive and the domain being included in the domain
$\brace{\norm{x} \geq t_0}$ and containing the domain
$\brace{\norm{x} \geq \sqrt{2}t_0}$, we get
\begin{align*}
  \int_{\brace{\norm{x}\geq \sqrt{2}t_0}}^\infty \exp(-A\norm{x}^2)\diff x
  \leq \paren{2\int_{t_0}^\infty \exp(-At^2)\diff t}^2
  \leq \int_{\brace{\norm{x}\geq t_0}}^\infty \exp(-A\norm{x}^2)\diff x.
\end{align*}
Using polar coordinate we get
\[
  \int_{\brace{\norm{x}\geq t_0}}^\infty \exp(-A\norm{x}^2)\diff x
  = 2\pi\int_{t_0}^\infty r\exp(-Ar^2)\diff r
  = \pi A^{-1} \exp(-At_0^2).
\]
Therefore
\[
  2^{-1}\sqrt{\pi} A^{-1/2} \exp(-A 2 t_0^2)^{1/2}
  \leq \int_{t_0}^\infty \exp(-At^2)\diff t
  \leq 2^{-1}\sqrt{\pi} A^{-1/2} \exp(-A t_0^2)^{1/2}.
\]
This explain the rates in $L$.
\end{proof}

\subsection{Proof of Theorem \ref{thm:low-density}}
\label{proof:low-density}

Using the calibration and Bernstein inequalities we get,
omitting the dependency of $L_n$ and $M_n$ to $n$,
\begin{align*}
  \E_{{\cal D}_n}{\cal R}(f_n) - {\cal R}(f^*)
  & \leq 2c_{\psi}\E_{{\cal D}_n, X}\bracket{\ind{d(g_n(X), g^*(X)) \geq d(g^*(X), F)}\norm{g_n(X) - g^*(X)}} 
  \\& = 2c_{\psi}\int_{0}^\infty \Pbb_{{\cal D}_n, X}\paren{\ind{d(g_n(X), g^*(X)) \geq d(g^*(X), F)}\norm{g_n(X) - g^*(X)} \geq t} \diff t 
  \\& = 2c_{\psi}\int_{0}^\infty \E_{X}\Pbb_{{\cal D}_n}\paren{\norm{g_n(X) - g^*(X)} \geq \max\brace{t, d(g^*(X), F)}} \diff t 
  \\& \leq 2c_{\psi}\int_{0}^\infty \E_{X}\exp\paren{-\frac{L\max\brace{t, d(g^*(X), F)}^2}{1 + M \max\brace{t, d(g^*(X), F)}^2}} \diff t 
  \\& = 2c_{\psi}\int_{0}^\infty \E_{X}\bracket{\ind{d(g^*(X), F) < t} \exp\paren{-\frac{Lt^2}{1 + Mt}}}\diff t 
  \\&\qquad + 2c_{\psi} \int_0^\infty \E_{X}\bracket{\ind{d(g^*(X), F) \geq t}\exp\paren{-L\frac{d(g^*(X), F)^2}{1 + Md(g^*(X), F)}}}\diff t 
  \\& = 2c_{\psi}\int_{0}^\infty \Pbb_{X}\paren{d(g^*(X), F) < t} \exp\paren{-\frac{Lt^2}{1 + Mt}}\diff t
  \\& \qquad + 2c_{\psi}\E_{X}\bracket{d(g^*(X), F) \exp\paren{-\frac{Ld(g^*(X), F)^2}{1 + Md(g^*(X), F)}}}.
\end{align*}
Let begin by working on the first term. We have, using the low-density
separation hypothesis
\begin{align*}
  \int_{0}^\infty \Pbb_{X}\paren{d(g^*(X), F) < t} \exp\paren{-\frac{Lt^2}{1 + Mt}}\diff t 
  &\leq c_{\alpha} \int_{0}^\infty t^\alpha \exp(-\frac{Lt^2}{1+Mt})\diff t.
\end{align*}
Recall the expression of the Gamma integral
\[
  \int_0^\infty t^\alpha \exp(-Lt) \diff t = \frac{\Gamma(\alpha+1)}{L^{\alpha+1}}
  \qquad\text{and}\qquad
  \int_0^\infty t^\alpha \exp(-Lt^2) \diff t = \frac{\Gamma\paren{\frac{\alpha+1}{2}}}{2 L^{\frac{\alpha+1}{2}}}.
\]
Let briefly talk about optimality. Up to now, we have only used three
inequality: calibration, concentration exponential inequality and low-density
separation. Therefore, when those inequalities hold as equalities, we get an
lower bound of order on the excess of risk as
\begin{align*}
  \E_{{\cal D}_n}{\cal R}(f_n) - {\cal R}(f^*)
  &\geq  2c_{\psi}c_{\alpha} \int_{0}^\infty t^\alpha \exp(-\frac{Lt^2}{1+Mt})\diff t
  \\&\geq  2c_{\psi}c_{\alpha} \int_{0}^\infty \frac{1}{2} t^\alpha \paren{\exp(-Lt^2) + \exp\paren{-\frac{Lt}{M}}}\diff t
  \\&= 2c_{\psi} c_{\alpha}\paren{\frac{\Gamma\paren{\frac{\alpha+1}{2}}}{4} L^{-\frac{\alpha+1}{2}}
  + \frac{\Gamma(\alpha+1)}{2} M^{\alpha+1}L^{-(\alpha + 1)}}.
\end{align*}
For the upper bound, using that $\exp(-a/1+b) \leq \exp(-a/2) + \exp(-a/2b)$, we get
\begin{align*}
  \int_{0}^\infty t^\alpha \exp(-\frac{Lt^2}{1+Mt})\diff t
  &\leq \int_{0}^\infty t^\alpha \exp(-\frac{Lt^2}{2})\diff t
  + \int_{0}^\infty t^\alpha \exp(-\frac{Lt}{2M})\diff t
  \\&= 2^{\frac{\alpha - 1}{2}} \Gamma\paren{\frac{\alpha+1}{2}} L^{-\frac{\alpha + 1}{2}} 
  + 2^{\alpha+1} \Gamma(\alpha + 1) M^{\alpha + 1} L^{-(\alpha + 1)} . 
\end{align*}

Let study the second term in the excess of risk inequality.
To enhance readability, write $\eta(X) = d(g^*(X), F)$. We will first
dissociate the two parts in the exponential with
\begin{align*}
  \E_{X}\bracket{\eta(X) \exp\paren{-\frac{L\eta(X)^2}{1 + M\eta(X)}}}
  \leq \E_{X}\bracket{\eta(X) \paren{\exp\paren{-\frac{L\eta(X)^2}{2}} + \exp\paren{-\frac{L\eta(X)}{2M}}}.}
\end{align*}
We are left with studying $\E[\eta(X) \exp(-A\eta(X)^p)]$, for $A > 0$ and
$p\in \brace{1, 2}$. The function $t\to~t\exp(-At^p)$ achieves its maximum in
$t_0 = (pA)^{-1/p}$, it is increasing before and decreasing after.
Notice that the quantity
\begin{align*}
  \Pbb(\eta(X) < t_0)\E_{X}\bracket{\eta(X) \exp(-A\eta(X)^p)\midvert \eta(X) < t_0}
  &\leq c_{\alpha} t_0^{\alpha + 1} \exp(-A t_0^p)
  \\&= c_{\alpha} p^{-\frac{\alpha + 1}{p}}\exp(-p^{-1/p}) A^{-\frac{\alpha + 1}{p}},
\end{align*}
is exactly of the same order as the control we had on the first term in
the excess of risk decomposition.
This suggests to consider the following decomposition
\begin{align*}
  &\E_{X}\bracket{\eta(X)\exp(-A\eta(X)^p)}= 
  \Pbb(\eta(X) < t_0)\E_{X}\bracket{\eta(X) \exp(-A\eta(X)^p)\midvert \eta(X) < t_0}
  \\&\qquad +\sum_{i=0}^\infty \Pbb(2^i t_0\leq \eta(X) < 2^{i+1}t_0)\E_{X}\bracket{\eta(X) \exp(-A\eta(X)^p)\midvert 2^it_0 \leq \eta(X) < 2^{i+1}t_0}
  \\&\qquad \leq c_{\alpha} t_o^{\alpha+1} \exp(-At_0^p)
  + \sum_{i=0}^\infty c_{\alpha} 2^\alpha (2^i t_0)^{\alpha + 1} \exp(-At_0^p (2^i)^p)
  \\&\qquad = c_{\alpha} t_o^{\alpha+1} \paren{\exp(-p^{-1/p})
  + \sum_{i=0}^\infty 2^\alpha 2^{i(\alpha + 1)} \exp(-p^{-1/p} 2^{ip})}.
\end{align*}
The convergence of the last series, ensures the existence of a constant $c$ such
that
\begin{align*}
  \E_{X}\bracket{\eta(X) \exp\paren{-\frac{L\eta(X)^2}{1 + M\eta(X)}}}
  &\leq c \paren{\paren{\frac{L}{2M}}^{-(\alpha + 1)}  + \paren{\frac{L}{2}}^{-\frac{\alpha + 1}{2}}}.
\end{align*}
Adding everything together, we get the existence of two constants
$c', c''$,
such that
\[
  \E_{{\cal D}_n}{\cal R}(f_n) - {\cal R}(f^*) \leq 2c_{\psi}c_{\alpha}\paren{c' M^{\alpha+1}
  L^{-(\alpha+ 1)} + c'' L^{-\frac{\alpha + 1}{2}}}.
\]
This ends the proof by considering $c = \max\paren{c', c''}$.

\subsection{Refinement of Theorem \ref{thm:low-density}}
\label{app:ref-low-density}

Some convergence analyses lead to exponential inequalities that are not of Bernstein
type, indeed, our result still holds in those settings, as mentioned by the following lemma.
In the following, we omit the dependency of $L_n$ and $M_n$ to $n$.

\begin{lemma}[Refinement of Theorem \ref{thm:low-density}]
  \label{lem:ref-low-density}
  Under the assumptions of Theorem \ref{thm:low-density}, if the concentration
  is not given by Assumption \ref{ass:concentration} but given, for some
  positive constants $(a_{i}, b_{i}, p_{i})_{i\leq m}$, by, for all
  $x\in\supp\rho_{\X}$ and $t > 0$, 
  \[
    \Pbb_{{\cal D}_n}(\norm{g_n(x) - g^*(x)} > t) \leq \sum_{i=1}^n a_{i} \exp(-b_{i}t^{p_{i}}).
  \]
  Then the excess of risk is controlled by
  \[
    \E_{{\cal D}_n}{\cal R}(f_n) - {\cal R}(f^*) \leq c\sum_{i=1}^n a_{i} b_{i}^{-\frac{\alpha + 1}{p_{i}}},
  \]
  for a constant $c$ that does not depend on $(a_{i}, b_{i})_{i\leq m}$.
\end{lemma}
\begin{proof}
  First of all, remark that the proof of Theorem \ref{thm:low-density} is linear
  in $\Pbb_{{\cal D}_n}(\norm{g_n(x) - g^*(x)} > t)$, therefore we only need to
  prove this lemma for $(a, b, p)$, for which we proceed as in Theorem \ref{thm:low-density}
  \begin{align*}
    \E_{{\cal D}_n}{\cal R}(f_n) - {\cal R}(f^*)
    & \leq 2c_{\psi} \E_{{\cal D}_n, X}\bracket{\ind{\norm{g_n(X) - g(X)} \geq d(g(X), F)}\norm{g_n(X) - g(X)}} 
    \\& = 2c_{\psi}\int_{0}^\infty \Pbb_{{\cal D}_n, X}\paren{\ind{\norm{g_n(X) - g(X)} \geq d(g(X), F)}\norm{g_n(X) - g(X)} \geq t} \diff t 
    \\& = 2c_{\psi}\int_{0}^\infty \E_{X}\Pbb_{{\cal D}_n}\paren{\norm{g_n(X) - g(X)} \geq \max\brace{t, d(g(X), F)}} \diff t 
    \\& \leq 2c_{\psi}a \int_{0}^\infty \E_{X}\exp\paren{- b\max\brace{t, d(g(X), F)}^p} \diff t 
    \\& = 2c_{\psi}a\int_{0}^\infty \E_{X}\bracket{\ind{d(g(X), F) < t} \exp\paren{-bt^p}}\diff t 
    \\&\qquad\qquad + 2c_{\psi}a \int_0^\infty \E_{X}\bracket{\ind{d(g(X), F) \geq t}\exp\paren{-bd(g(X), F)^p}}\diff t 
    \\& = 2c_{\psi}a\int_{0}^\infty \Pbb_{X}\paren{d(g(X), F) < t} \exp\paren{-bt^p}\diff t
    \\&\qquad\qquad+ 2c_{\psi}a\E_{X}\bracket{d(g(X), F) \exp\paren{-bd(g(X), F)^p}}.
  \end{align*}
  Let begin by working on the first term. We have, using the low-density
  separation hypothesis
  \begin{align*}
    &\int_{0}^\infty \Pbb_{X}\paren{d(g(X), F) < t} \exp\paren{-bt^2}\diff t 
    \leq c_{\alpha} \int_{0}^\infty t^\beta \exp(-bt^p)\diff t.
    \\ &\qquad\qquad= b^{-\frac{1+\beta}{p}} c_{\alpha} \int_{0}^\infty (b^{1/p}t)^\beta \exp(-(b^{1/p}t)^p)\diff (b^{1/p}t).
    \\ &\qquad\qquad= b^{-\frac{1+\beta}{p}} c_{\alpha} \int_{0}^\infty t^\beta \exp(-t^p)\diff t = c_{\alpha} \Gamma(\beta, p) b^{-\frac{1+\beta}{p}}.
  \end{align*}
  Let study the second term in the excess of risk inequality.
  To enhance readability, write $\eta(X) = d(g(X), F)$.
  We are left with studying $\E[\eta(X) \exp(-b\eta(X)^p)]$.
  The function $t\to~t\exp(-bt^p)$ achieves it maximum in $t_0 = (pb)^{-1/p}$, it is increasing before and decreasing after.
  Notice that the quantity
  \begin{align*}
    \Pbb(\eta(X) < t_0)\E_{X}\bracket{\eta(X) \exp(-b\eta(X)^p)\midvert \eta(X) < t_0}
    &\leq c_{\alpha} t_0^{\beta + 1} \exp(-b t_0^p)
    \\&= c_{\alpha} p^{-\frac{\beta + 1}{p}}\exp(-p^{-1/p}) b^{-\frac{\beta + 1}{p}},
  \end{align*}
  is exactly of the same order as the control we had on the first term in
  the excess of risk decomposition.
  This suggests to consider the following decomposition
  \begin{align*}
    &\E_{X}\bracket{\eta(X)\exp(-b\eta(X)^p)}
    = 
      \Pbb(\eta(X) < t_0)\E_{X}\bracket{\eta(X) \exp(-b\eta(X)^p)\midvert \eta(X) < t_0}
    \\&\quad\qquad +\sum_{i=0}^\infty \Pbb(2^i t_0\leq \eta(X) < 2^{i+1}t_0)\E_{X}\bracket{\eta(X) \exp(-b\eta(X)^p)\midvert 2^it_0 \leq \eta(X) < 2^{i+1}t_0}
    \\&\qquad \leq c_{\alpha} t_o^{\beta+1} \exp(-bt_0^p)
    + \sum_{i=0}^\infty c_{\alpha} 2^\beta (2^i t_0)^{\beta + 1} \exp(-bt_0^p (2^i)^p)
    \\&\qquad = c_{\alpha} t_o^{\beta+1} \paren{\exp(-p^{-1/p})
    + \sum_{i=0}^\infty 2^\beta 2^{i(\beta + 1)} \exp(-p^{-1/p} 2^{ip})}.
  \end{align*}
  The convergence of the last series ensures the existence of a constant $c'$ such
  that
  \begin{align*}
    \E_{X}\bracket{\eta(X) \exp(-b\eta(X)^p)}
    \leq c' b^{-\frac{\beta+1}{p}}.
  \end{align*}
  Adding everything together ends the proof of this lemma.
  Note that we have the same type of optimality as the one stated in Theorem \ref{thm:low-density}.
\end{proof}

Because we use concentration inequalities for terms that are not necessarily
centered, we usually get that Eq. \eqref{eq:concentration} only holds for $t > \epsilon_0$
where, typically $\epsilon_0 = \norm{\E_{{\cal D}_n}g_n(x) - g^*(x)}$, we can bypass
this problem by adding in $\ind{t<\epsilon_0}$ in the probability, motivating the
study leading to the following lemma.

\begin{lemma}[Handling bias in concentration inequality]
  Under the assumptions of Theorem \ref{thm:low-density}, if the concentration
  is not given by Assumption \ref{ass:concentration} but given, for a $\epsilon_0 > 0$, by, for all
  $x\in\supp\rho_{\X}$ and $t > 0$, 
  \[
    \Pbb_{{\cal D}_n}(\norm{g_n(x) - g^*(x)} > t) \leq \ind{t<\epsilon_0}.
  \]
  Then the excess of risk is controlled by
  \[
    \E_{{\cal D}_n}{\cal R}(f_n) - {\cal R}(f^*) \leq 2c_{\psi}c_{\alpha} \epsilon_0^{\alpha + 1}.
  \]
\end{lemma}
\begin{proof}
  We retake the beginning of the proof of Theorem \ref{thm:low-density}, and change its ending with
  \begin{align*}
    \E_{{\cal D}_n}{\cal R}(f_n) - {\cal R}(f^*)
    & \leq 2c_{\psi} \E_{{\cal D}_n, X}\bracket{\ind{\norm{g_n(X) - g(X)} \geq d(g(X), F)}\norm{g_n(X) - g(X)}} 
    \\& = 2c_{\psi}\int_{0}^\infty \Pbb_{{\cal D}_n, X}\paren{\ind{\norm{g_n(X) - g(X)} \geq d(g(X), F)}\norm{g_n(X) - g(X)} \geq t} \diff t 
    \\& = 2c_{\psi}\int_{0}^\infty \E_{X}\Pbb_{{\cal D}_n}\paren{\norm{g_n(X) - g(X)} \geq \max\brace{t, d(g(X), F)}} \diff t 
    \\& \leq 2c_{\psi} \int_{0}^\infty \E_{X} \ind{t<\epsilon_0}\ind{d(g(X), F) < \epsilon_0} \diff t 
    = 2c_{\psi} \epsilon_0 \Pbb_{X}\paren{d(g(X), F) <\epsilon_0} \diff t.
  \end{align*}
  This leads to the result after applying the $\alpha$-margin condition.
\end{proof}

\section{Nearest neighbors}
\subsection{Usual assumptions to derive nearest neighbors convergence rates}
\label{app:nn-assumptions}

Assumption \ref{ass:lipschitz} can be seen as the backbone that allow to control
$\norm{g_n^*(x) - g^*(x)}$ in a uniform manner. This assumption that relates
the regularity of $g^*$ with the density of $\rho_{\X}$ has been historically
approached in the following manner. 
Assume that $g^*$ is $\beta'$-H\"older, that is, for any $x, x' \in \supp\rho_{\X}$
\[
  \norm{g^*(x) - g^*(x')} \leq a_1 d(x, x')^{\beta'}.
\]
Suppose that $\X = \R^d$, that $\rho_{\X}$ is continuous against $\lambda$, the Lebesgue
measure, with minimal mass in the sense that there exists a $p_{\min} > 0$ such
that $\frac{\diff\rho_{\X}}{\diff\lambda}(\X)$ does not intersect $(0, p_{\min})$,
and that $\supp\rho_{\X}$ has regular boundaries in the sense that there exist
$a_2, t_0 > 0$ such that for any $x\in\supp\rho_{\X}$ and $t \in (0, t_0)$
\[
  \lambda\paren{{\cal B}(x, t)\cap\supp\rho_{\X}} \geq a_2\lambda\paren{{\cal B}(x, t)}.
\]
For example an orthant satisfies this property with $a_2 = 2^{-d}$ and $t_0 =
\infty$, and ${\cal B}(0, 1)$ satisfies this property with 
$a_2 = \lambda\paren{{\cal B}(0, 1) \cap {\cal B}(1, 1)} / \lambda\paren{{\cal
    B}(0, 1)}$ and $t_0 = 1$.
In such a setting, we get
\[
  \norm{g^*(x) - g^*(x')} \leq a_1 d(x, x')^{\beta'}
  = a_1 \paren{\frac{\lambda({\cal B}(x, d(x, x')))}{\lambda({\cal B}(0, 1))}}^{\frac{\beta'}{d}}.
\]
Where $d(x, x') < t_0$, we have
\[
  \lambda({\cal B}(x, d(x,x'))) \leq a_2^{-1} \lambda\paren{{\cal B}(x, d(x,
    x'))\cap \supp\rho_{\X}}
  \leq a_2^{-1}p_{\min}^{-1} \rho_{\X}({\cal B}(x, d(x,x')))^\beta.
\]
This means that for any $x \in \supp\rho_{\X}$ and $x' \in {\cal B}(x, t_0)$ we
have, with $\beta = \frac{\beta'}{d}$ and the constant
$a_3= a_1 a_2^{-\beta} p_{\min}^{-\beta}\lambda({\cal B}(0,1))^{-\beta}$
\[
  \norm{g^*(x) - g^*(x')} \leq a_3 \rho_{\X}({\cal B}(x, d(x, x')))^{\beta}.
\]
While, we actually do not need the bound to hold for $d(x, x') > t_0$ in the
following proof, to check the veracity of our remark on Assumption
\ref{ass:lipschitz},
one can verify that under our assumptions on $\rho_{\X}$, $\supp\rho_{\X}$ is
bounded, and therefore $g^*$ is too. And if $g^*$ is bounded by $c_\phi$, by considering  
$a_3' = \max\paren{2c_\phi a_2^{-\beta}p_{\min}^{-\beta} t_0^{-\beta'}, a_3}$,
this bound holds for any $x, x' \in \supp\rho_{\X}$.

\subsection{Proof of Lemma \ref{lem:nn-concentration}}
\label{proof:nn-concentration}

\paragraph{Control of the variance term.}
For $x\in\rho_{\X}$, the variance term can be written
\[
  \norm{g_n(x) - g_n^*(x)} = \norm{\frac{1}{k}\sum_{i=1}^k \phi(Y_{(i)}) -
    \E\bracket{Y_{(i)}\midvert X_{(i)}}}.
\]
Where the index $(i)$ is such that $X_{(i)}$ is the $i$-th nearest neighbor of
$x$ in $(X_i)_{i\leq n}$. Since, given $(X_i)_{i\leq n}$, the $(Y_i)_{i\neq
  n}$ are independent, distributed according to $\otimes_{i\leq
  n}\rho\vert_{X_i}$, we can use a concentration inequality to control it.
We recall Bernstein concentration inequality in such spaces, derived by
\citet{Yurinskii1970}, we will use the formulation of Corollary 1 from \citet{Pinelis1986}. 

\begin{theorem}[Concentration in Hilbert space \citep{Pinelis1986}]
  \label{thm:bernstein-vector-full}
  Let denote by ${\cal A}$ a Hilbert space and by $(\xi_i)$ a sequence of independent
  random vectors on ${\cal A}$ such that $\E[\xi_i] = 0$, and that there exists
  $M, \sigma^2 > 0$ such that for any $m \geq 2$
  \[
    \sum_{i=1}^n \E\bracket{\norm{\xi_i}^m} \leq \frac{1}{2} m!\sigma^2 M^{m-2}.
  \]
  Then for any $t>0$
  \[
    \Pbb(\norm{\sum_{i=1}^n \xi_i} \geq t) \leq
    2\exp\paren{-\frac{t^2}{2\sigma^2 + 2tM}}.
  \]
\end{theorem}
This explain Assumption \ref{ass:moment}, allowing, because
there is only $k$ $\xi_i$ active in $\sum_{i=1}^n \alpha_i(x)\xi_i$, to get
\[
  \Pbb_{{\cal D}_n}\paren{\norm{g_n(x) - g_n^*(x)} > t}
  \leq 2\exp\paren{-\frac{kt^2}{2\sigma^2 + 2M t}}.
\]

\paragraph{Control of the bias term.}
Under the Modified Lipschitz condition, Assumption \ref{ass:lipschitz},
\begin{align*}
  \norm{g_n^*(x) - g_n(x)} &= \norm{\sum_{i=1}^n \alpha_i(x) \paren{g_n(x) - g^*(X_i)}}
  \leq \sum_{i=1}^n \alpha_i(x) \norm{g_n(x) - g^*(X_i)}
  \\&\leq c_{\beta} \sum_{i=1}^n \alpha_i(x) \rho_{\X}\paren{{\cal B}(x, d(x, X_i))}^\beta
  \leq c_{\beta} \rho_{\X}\paren{{\cal B}(x, d(x, X_{k}(x)))}^\beta.
\end{align*}
When $\rho_{\X}$ is continuous, it follows from the probability integral transform (also known as universality of the uniform)
that $\rho_{\X}\paren{{\cal B}(x, d(x, X_{k}(x)))}$ behaves like the $k$-th order
statistics of a sample $(U_i)_{i\leq n}$ of $n$ uniform distributions on
$[0,1]$.
Therefore, for any $s\in[0,1]$
\begin{align*}
  \Pbb_{{\cal D}_n}\paren{\rho_{\X}\paren{{\cal B}(x, d(x, X_{k}(x)))} > s}
  &= \Pbb\paren{\sum_{i=1}^n \ind{U_i < s} \leq k}.
\end{align*}
Recall the multiplicative Chernoff bound, stating that for $(Z_i)_{i\leq n}$
$n$ independent random variables in $\brace{0,1}$, if $Z = \sum_{i=1}^n Z_i$, and
$\mu = \E[Z]$, for any $\delta > 0$
\[
  \Pbb\paren{Z \leq (1-\delta)\mu} \leq \exp\paren{-\frac{\delta^2\mu}{2}}.
\]
Since, for $s \in [0, 1]$, $\E[\ind{U_i < s}] = \Pbb(U_i < s) = s$, we get,
when $k \leq ns / 2$
\[
  \Pbb\paren{\sum_{i=1}^n \ind{U_i < s} \leq k} \leq
  \exp\paren{-\frac{\paren{ns - k}^2}{2 ns}}
  \leq \exp\paren{-\frac{ns}{8}}.
\]
With $s = c_{\beta}^{-1}t^{\frac{1}{\beta}}$, we get
\[
  \Pbb_{{\cal D}_n}\paren{\norm{g_n^*(x) - g_n(x)} > t}
  \leq \exp\paren{-\frac{nt^{\frac{1}{\beta}}}{8c_{\beta}}}.
\]
Remark that when $g^*$ is $\beta'$ H\"older, we get the same result with
$\rho_{\X}({\cal B}(x, t))$ instead of $t^{\frac{1}{\beta}}$ by considering
$\ind{X_i \in {\cal B}(x, t)}$ instead of $\ind{U_i \leq t}$.
Note that there is way to bound a Binomial distribution with a Gaussian for $t$
smaller than the mean of the binomial distribution, which would allow to get
a bound that holds for any $t>0$ \citep{Slud1977}.  

\subsection{Proof of Theorem \ref{thm:nn-no-density}}
\label{proof:nn-no-density}
Using the proof of Theorem \ref{thm:no-density}, we get
\[
  \E_{{\cal D}_n}{\cal R}(f_n) - {\cal R}(f^*)
  \leq \ell_\infty\Pbb_{{\cal D}_n}\paren{\norm{g(x) - g_n^*(x)} > t_0}.
\]
Because $\norm{g_n(x) - g^*(x)} > t_0$ implies that either
$\norm{g_n(x) - g_n^*(x)} > t_0/2$ or $\norm{g_n^*(x) - g^*(x)} > t_0/2$, we get using Lemma
\ref{lem:nn-concentration}  
\[
  \Pbb_{{\cal D}_n}\paren{\norm{g(x) - g_n^*(x)} > t_0}
  \leq
  2\exp\paren{-\frac{b_1kt_0^2}{4 + 2b_2t_0}}
  + \exp\paren{-2^{-\frac{1}{\beta}}b_3 nt_0^{\frac{1}{\beta}}}
  +\ind{t_0 > \paren{k/2n}^\beta}.
\]
This explains the result of Theorem \ref{thm:nn-no-density}.

\subsection{Proof of Theorem \ref{thm:nn-low-density}}
\label{proof:nn-low-density}

First of all for $t > 0$, and $x \in \supp\rho_{\X}$, because $\norm{g_n(x) -
  g^*(x)} >t$ implies that either $\norm{g_n(x) - g_n^*(x)} > t / 2$ or 
$\norm{g_n^*(x) - g^*(x)} > t / 2$, we have the inclusion of events:
\[
  \brace{{\cal D}_n \midvert \norm{g_n(x) - g^*(x)} > t}
  \subset \brace{{\cal D}_n \midvert \norm{g_n(x) - g^*(x)} > t/2}
  \cup \brace{{\cal D}_n \midvert \norm{g_n(x) - g^*(x)} > t/2},
\]
which translates in term of probability as
\begin{align*}
  \Pbb_{{\cal D}_n}\paren{\norm{g_n(x) - g^*(x)} > 2t}
  &\leq \Pbb_{{\cal D}_n}\paren{\norm{g_n(x) - g^*(x)} > t}
  +\Pbb_{{\cal D}_n}\paren{ \norm{g_n(x) - g^*(x)} > t}.
  \\&\leq 2\exp\paren{-\frac{b_1kt^2}{1 + b_2t}}
  + \exp\paren{-b_3 nt^{\frac{1}{\beta}}}
  + \ind{t < \paren{\frac{k}{2n}}^\beta}.
\end{align*}
Using the refinements of Theorem \ref{thm:low-density} exposed in Appendix
\ref{app:ref-low-density},
we get that there exists a constant $c > 0$ that does not depend on $k$ or $n$ such that 
\[
  \E_{{\cal D}_n}{\cal R}{f_n} - {\cal R}(f^*)
  \leq 
  c \paren{k^{-\frac{\alpha + 1}{2}} + n^{-\beta(\alpha + 1)} +
    (nk^{-1})^{\beta(\alpha + 1)}}. 
\]
We optimize this last quantity with respect to $k$ by taking $k = n^{\gamma}$,
and choosing $\gamma$ such that
$\gamma = 2(1-\gamma)\beta$
leading to $\gamma = 2\beta / (2\beta + 1)$ and to rates in $n$ to the power
minus $\beta(\alpha + 1) / (2\beta + 1)$.

\subsection{Numerical Experiments}

\begin{figure}[ht]
  \centering
  \includegraphics{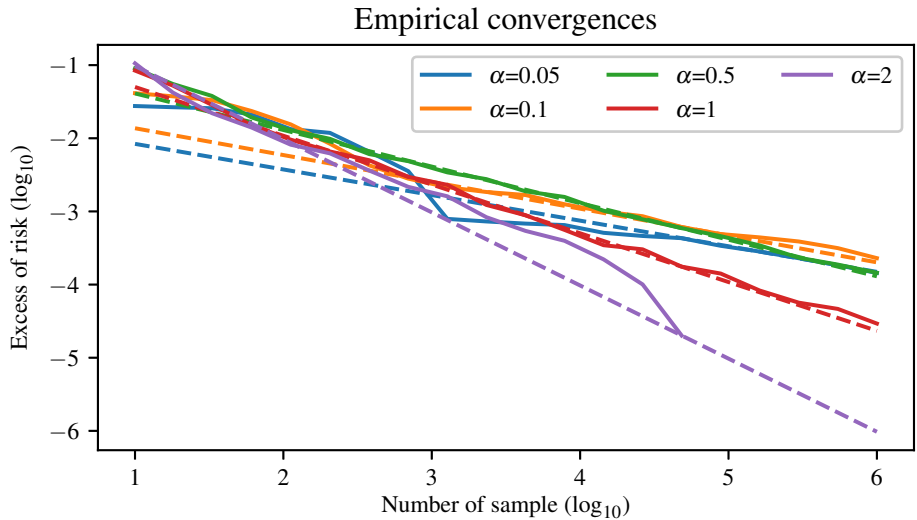}
  \caption{Supplement to Figure \ref{fig:rates}. We precise the error is
    evaluated on 100 points forming a regular partition of $\X = [-1, 1]$, and
    the expectation $\E_{{\cal D}_n}$ is approximated by considering 100
    datasets. The violet curve is cropped at $n\approx 10^5$, because the error
    was null afterwards with our evaluation parameters (only 100 points to
    evaluate the  error), forbidding us  to consider the logarithm of the excess
    of risk.} 
  \label{fig:more_rates}
\end{figure}

Interestingly, on numerical simulations such as the one presented on Figure
\ref{fig:rates}, we observed two regimes. A first regime where bound are
meaningless because of constants being too big, and where the error decreases
independently of the exponent expected through Theorem \ref{thm:nn-low-density},
and a final regime where rates corresponds to the bond given by the theorem. 
Note that when $\alpha >> 1$, with our computation parameter, we do not get to
really illustrate convergence rates, as this final regime get place for bigger
$n$ than what we have considered ($n_{\max} = 10^6$), this being partly due to the
constant $c_\beta$ in Assumption \ref{ass:lipschitz} being big for the
$g^*$ we considered. Furthermore, note that, for example, if a problem satisfied
Assumption \ref{ass:no-density} with a really small $t_0$, we expect that
exponential convergence rates are only going to be observed for $n > N$, with
$N$ really big, and for which the excess of risk is already really small.


\section{Kernel proofs}
\label{app:rkhs}
In this section, we study $L^\infty$ convergence rates of the kernel ridge
regression estimate.
We use the $L^2$-proof scheme of \citet{Caponnetto2007} with the remark of
\citet{Ciliberto2016} to factorize the action of $K$ on $L^2(\X, {\cal H},
\rho_{\X})$ through its action on $L^2(\X, \R, \rho_{\X})$.
We retake the work of \citet{PillaudVivien2018} to relax the source condition,
and use \citet{Fischer2020} to cast in $L^\infty$ thanks to interpolation inequality.
While those results, leading to Lemma \ref{lem:rkhs}, are not new, we present them
entirely to provide the reader with self-contained materials.

\subsection{Construction of reproducing kernel Hilbert space (RKHS)}
\label{app:rkhs-construction}

In the following, we suppose $k$ bounded by $\kappa^2$.

\paragraph{Vector-valued RKHS.}
To study the estimator $g_{n}$, it is useful to introduce the reproducing kernel
Hilbert space ${\cal G}$ associated with $k$ and ${\cal H}$ 
\citep{Aronszajn1950}.
To define ${\cal G}$, define the atoms $k_{x}:{\cal H}\to{\cal G}$ and the
scalar product, for $x, x' \in \X$ and $\xi, \xi' \in {\cal H}$ as
\[
  \scap{k_{x}\xi}{k_{x'}\xi'}_{\cal G} = \scap{\xi}{\Gamma(x, x')\xi'} =
  k(x, x')\scap{\xi}{\xi'}_{\cal H}.
\]
Where $\Gamma$ is the vector valued kernel inherit from $k$ as $\Gamma(x, x') = k(x, x') I_{\cal H}$ \citep{Schwartz1964}.
${\cal G}$ is defined as the closure, under the metric induced by this scalar
product, of the span of the atoms $k_{x} \xi$ for $x\in\X$ and $\xi\in{\cal H}$.
Note that $k_{x}$ is linear, and continuous of norm $\norm{k_{x}}_{\op} = \sqrt{k(x, x)}$.
When $k(\cdot, x)$ is square integrable for all $x\in\supp\rho_{\X}$, ${\cal G}$
is homomorphic to a functional space in $L^2(\X, {\cal H}, \rho_{\X})$ through
the linear mapping $S$ that associates the atom $k_{x}\xi$ in ${\cal G}$ to
the function $k(\cdot, x)\xi$ in $L^2$, defined formally as
\myfunction{S}{\cal G}{L^2}{\gamma}{x\to k_{x}^{\star} \gamma.}
While intrinsically similar, it is useful to distinguish between $\cal G$ and
$\ima S \subset L^2$.
Note that $S$ is continuous, since on atom $k_x \xi$, $\norm{Sk_x\xi}_{L^2} \leq \norm{k_x(\cdot)}_{L^2} \norm{\xi}_{\cal H} \leq k(x,x)\norm{\xi}_{\cal H} = \norm{k_x\xi}_{\cal G}$. The fact that $S$ is a bounded operator justifies the introduction of the following operators.

\paragraph{Central operators.}
In the following, we will make an extensive use of
$S^{\star}:L^2(\X, {\cal H}, \rho_{\X})\to{\cal G}$ the
adjoint of $S$, defined as $S^{\star} g = \E_{\rho_{\X}}[k_{X}g(X)]$;
the covariance operator $\Sigma:{\cal G}\to{\cal G}$, defined as
$\Sigma := S^{\star} S = \E_{\rho_{\X}}[k_{X} k_{X}^{\star}]$; and its action on $L^2$,
$K:L^2(\X, {\cal H}, \rho_{\X}) \to L^2(\X, {\cal H}, \rho_{\X})$, defined as
$Kg := SS^{\star} g = \E_{X}[k(\cdot, X)g(X)]$. Finally, we have define the four
central operators
\begin{equation}
  \begin{array}{ll}
    S\gamma = k_{(\cdot)}^{\star} \gamma, & S^{\star} g = \E_{\rho_{\X}}[k_{X}g(X)]\\
    \Sigma := S^{\star} S = \E_{\rho_{\X}}[k_{X} k_{X}^{\star}],\qquad & Kg := SS^{\star} g = \E_{X}[k(\cdot, X)g(X)].
  \end{array}  
\end{equation}
It should be noted that this construction is usually avoided since, based on the fact that the Frobenius norm of $K$ behave like $\dim({\cal H})$, meaning that when $\cal H$ is infinite dimensional, $K$ is not a compact operator. However, since we consider $\cal Z$ finite, we can always consider ${\cal H} = \R^{\card{\cal Z}}$ with $\phi(y) = (\ell(z, y)_{z\in\cal Z}$ and $\psi(z) = (\ind{z=z'})_{z'\in{\cal Z}}$, and moreover, we will see that a way can be worked out, even when $\cal H$ is infinite dimensional, which was already shown by \citet{Ciliberto2016}.

\subsubsection*{Relation between real-valued versus vector-valued RKHS.}
Usually convergence in RKHS are studied for real-valued function. We need
convergence results for vector-valued function. As mentioned above, we only need the results for Euclidean space, however, we will do it for function that are maps going into potentially infinite dimensional Hilbert space.
Indeed, this does not lead to major complication. We provide here one way to get
around this issue.
An alternative formal way to proceed can be found \citep{Ciliberto2016}.

\paragraph{Real-valued RKHS.}
We build the real-valued RKHS ${\cal G}_{\X}$ as the closure of the
span of the atoms $\bar{k}_{x}$ for $x\in\X$, under the metric induced by the scalar
product $\scap{\bar{k}_{x}}{\bar{k}_{x'}} = k(x, x')$.
Similarly, we build $\bar{S}$, $\bar{S}^{\star}$, $\bar\Sigma$ and $\bar{K}$.
We shall see that the action of $\Sigma$ on ${\cal G}$ can be factorized through
its actions $\bar\Sigma$ on ${\cal G}_{\X}$.

\paragraph{Algebraic equivalences.}
Based on the fact that $\norm{\bar{k}_{x}}_{{\cal G}_{\X}} = \norm{k_{x}}_{\op} =
\sqrt{k(x, x)}$, it is possible to build an isometry that match $\bar{k}_{x}$ in
${\cal G}_{\X}$ to $k_{x}$ in the space of continuous linear operator from ${\cal H}$
to ${\cal G}$. With $(\bar{e}_{i})_{i\in \N}$ an orthogonal basis of ${\cal
  G}_{\X}$, and $(f_{j})_{j\in\N}$ an orthogonal basis of ${\cal H}$, we get an
orthogonal basis $(e_{i}f_{j})_{i,j\in\N}$ of ${\cal G}$.
This is exaclty the construction ${\cal G} = {\cal G}_{\X} \otimes {\cal H}$ of \citep{Ciliberto2016}.

Note that for $\mu_1, \mu_2$ two measures on $\X$, we can check that
\begin{align*}
  \norm{\E_{X\sim\mu_1}[k_{X}k_{X}^{\star}]\E_{X_0\sim\mu_2}[k_{X_0}]}_{\op}^2
  & = \norm{\E_{X\sim\mu_1}[\bar{k}_{X}\bar{k}_{X}^{\star}]\E_{X_0\sim\mu_2}[\bar{k}_{X_0}]}_{{\cal G}_{\X}}^2
  \\&=\E_{X, X'\sim\mu_1; X, X'\sim\mu_2} [k(X_0, X)k(X, X')k(X', X'_0)].
\end{align*}
This explains that we will allow ourselves to write derivations of the type
\[
  \norm{(\Sigma+\lambda)^{-\frac{1}{2}} k_{x} g_{n}(x)}_{\cal G} \leq
  \norm{(\bar\Sigma+\lambda)^{-\frac{1}{2}} \bar{k}_{x}}_{{\cal G}_{\X}}
  \norm{g_{n}(x)}_{\cal H}.
\]

Note also that for $g := \sum_{ij} c_{ij} e_{i}f_{j} \in {\cal G}$, with
$\sum_{ij}c_{ij}^2 = 1$, $\bar{c}_{i} := (c_{ij})_{j\in\N} \in \ell^2$, $\bar{A}$
an self-adjoint operator on ${\cal G}_{\X}$ and $A$ its version on ${\cal G}$, we
have
\begin{align*}
  \norm{Ag}^2_{{\cal G}}
  &= \sum_{ijk} c_{ij} c_{kj} \scap{\bar{A}\bar{e}_{i}}{\bar{A}\bar{e}_k}_{\cal G}
  = \sum_{ij} \scap{\bar{c}_{i}}{\bar{c}_{j}}_{\ell^2} \scap{\bar{A}\bar{e}_{i}}{\bar{A}\bar{e}_k}_{{\cal G}_{\X}}
  \\& \leq \sum_{ij}\norm{\bar{c}_{i}}_{\ell^2}\norm{\bar{c}_{j}}_{\ell^2} \scap{\bar{A}\bar{e}_{i}}{\bar{A}\bar{e}_k}_{{\cal G}_{\X}}
  = \norm{\bar{A} \sum_{i} \norm{\bar{c}_{i}}_{\ell^2} \bar{e}_{i}}^2_{{\cal G}_{\X}}
  \leq \norm{\bar{A}}_{\op}^2,
\end{align*}
which explains why we will consider derivations of the type
\[
  \norm{(\Sigma+\lambda)^{\frac{1}{2}}(\hat\Sigma+\lambda)^{-1}(\Sigma+\lambda)^{\frac{1}{2}}}_{\op}
  \leq \norm{(\bar\Sigma+\lambda)^{\frac{1}{2}}(\hat{\bar\Sigma}+\lambda)^{-1}(\bar\Sigma+\lambda)^{\frac{1}{2}}}_{\op}.
\]
Finally, notice that because of the same consideration, if $(\bar{u}_i)_{i\in\N} \in {\cal G}_\X^\N$ diagonalize $\bar{A}$, $(u_if_j)_{i,j\leq \N}\in{\cal G}^{\N\times\N}\eqsim{\cal G}^\N$ diagonalize $A$ in ${\cal G}$. This justifies the consideration of fractional operators $A^p$ for $p \in \R_+$, such as in Assumptions \ref{ass:interpolation} and \ref{ass:source}.
Based on those equivalence, we will forget the bar notations, we incite the
careful and attentive reader to recover them.

\subsection{Estimate \texorpdfstring{$g_{n}$}{} as an empirical approximate projection on RKHS}
To obtain bounds like Eq. \eqref{eq:concentration}, it is sufficient to control the
convergence of $g_{n}$ to $g^{*}$ in $L^\infty$.
Assumption \ref{ass:interpolation} allow us to cast in $L^2$ the study of the
convergence in $L^\infty$.
The convergence of $g_{n}$ towards $g^{*}$ can be split in two terms, a term
expressing the convergence of $g_\lambda$ towards $g^{*}$ that is based on
geometrical properties and a term expressing the convergence of $g_{n}$ towards
$g_\lambda$, that is based on concentration inequalities in ${\cal G}$, such as
the ones given by \citet{Pinelis1986,Minsker2017}.
For this last term, we need to characterize $g_{n}$ and $g_\lambda$ with the
following lemma.

\begin{lemma}[Approximation of integral operators]
  $g_{n}$ can be understood as the empirical approximation of $g_\lambda$
  since
  \[
    g_{n} = S(\E_{\hat\rho}[k_{X} k_{X}^{\star}] + \lambda)^{-1}
    \E_{\hat\rho}[k_{X}\phi(Y)],
    \qquad
    g_\lambda = S(\E_{\rho}[k_{X} k_{X}^{\star}] + \lambda)^{-1}
    \E_{\rho}[k_{X}\phi(Y)],
  \]
  with $\hat\rho = \frac{1}{n} \sum_{i=1}^{n} \delta_{X_{i}} \otimes \delta_{Y_{i}}$,
\end{lemma}
\begin{proof}
  Indeed, the expression of $g_{n}$ and its convergences towards $g^{*}$ will be
  understood thanks to the operator $S$ and its derivatives.
  When $\ima S$ is closed in $L^2$, on can defined the orthogonal projection of
  $g^{*}$ to $\ima S$, with the $L^2$ metric as
  $\pi_{\ima S}(g^{*}) = S(S^{\star} S)^\dagger S^{\star} g^{*}$.
  When $\ima S$ is not closed, or equivalently when $\Sigma$ has positive eigen values
  converging to zero, one can define approximate orthogonal projection, through
  eigen value thresholding or Tikhonov regularization. This last choice leads to
  the estimate
  \[
    g_\lambda = S(\Sigma + \lambda)^{-1}S^{\star} g^{*} = S (S^{\star} S + \lambda)^{-1} S^{\star} g^{*}
    = SS^{\star}  (SS^{\star} +\lambda)^{-1}g^{*} = K (K + \lambda)^{-1}g^{*}.
  \]
  Note that, because of the Bayes optimum characterization of $g^{*}$,
  $S^{\star} g^{*} = \E_{\rho}[k_{X} \phi(Y)]$.
  This explains the characterization of $g_\lambda$.

  Interestingly, the approximation of $\rho$ by $\hat\rho$ can be thought with the
  approximation of $L^2(\X, {\cal H}, \rho_{\X})$ by
  $\ell^2({\cal H}^{n}) \simeq L^2(\X, {\cal H}, \hat\rho_{\X})$ where for
  $\Xi = (\xi_{i}), Z = (\zeta_{i}) \in {\cal H}^{n}$,
  \[
    \scap{\Xi}{Z}_{\ell^2} = \frac{1}{n}\sum_{i=1}^{n} \scap{\xi_{i}}{\zeta_{i}}_{\cal H},
  \]
  and with the empirical probability measure
  $\hat\rho = \frac{1}{n}\sum_{i=1}^{n} \delta_{x_{i}}\otimes \delta_{y_{i}}$.
  We redefine the natural homomorphism of $\cal G$ into $\ell^2$ with
  \myfunction{\hat{S}}{\cal G}{\ell^2}{\gamma}{\paren{k_{x_{i}}^{\star} \gamma}_{i\leq n}.}
  We check that its adjoint is, for $\Xi\in{\cal H}^{n}$ and $\gamma\in{\cal G}$
  \[
    \scap{\hat{S}^{\star} \Xi}{\gamma}_{\cal G} = \scap{\Xi}{\hat S \gamma}_{\ell^2}
    = \frac{1}{n} \sum_{i=1}^{n}\scap{\xi_{i}}{k_{x_{i}}^{\star}\gamma}_{\cal H}
    = \scap{\frac{1}{n} \sum_{i=1}^{n} k_{x_{i}}\xi_{i}}{\gamma}_{\cal G}.
  \]
  Similarly we define $\hat K:{\cal H}^{n}\to{\cal H}^{n}$ and $\hat\Sigma:{\cal G}\to{\cal G}$,
  with
  \[
    \hat K\Xi = \hat S \hat S^{\star} \Xi = \paren{\frac{1}{n} \sum_{i=1}^{n} k(x_{j}, x_{i})
      \xi_{i}}_{j\leq n},\qquad
    \hat\Sigma = \frac{1}{n}\sum_{i=1}^{n} k_{x_{i}} \otimes k_{x_{i}} = \E_{\hat\rho_{\X}}[k_{X}k_{X}^{\star}].
  \]
  Finally we define $\hat{\Phi} = (\phi(y)_{i})_{i\leq n} \in {\cal H}^{n}$,
  so that
  \[
    \widehat{S^{\star} g^{*}} := \E_{\hat\rho}[\phi(Y)\cdot k_{X}] = \hat S^{\star} \hat\Phi.
  \]
  Finally we can express $g_{n}$ as
  \[
    g_{n} = S(\hat\Sigma + \lambda)^{-1}\hat S^{\star} \hat\Phi
    = S(\hat S^{\star} \hat S + \lambda)^{-1}\hat S^{\star} \hat\Phi 
    = S\hat S^{\star} (\hat S\hat S^{\star} + \lambda)^{-1} \hat\Phi 
    = S\hat S^{\star} (\hat K + \lambda)^{-1} \hat\Phi.
  \]
  This explains the equivalence between $g_{n}$ defined at the beginning of
  Section \ref{sec:rkhs} and the $g_{n}$ expressed in the lemma, that will be used for
  derivations of theorems. 
\end{proof}

\subsection{Linear algebra and equivalent assumptions to Assumptions \ref{ass:capacity},
  \ref{ass:interpolation}}

To proceed with the study of the convergence of $g_{n}$ towards $g_\lambda$ in
$L^2$, it is helpful to pass by ${\cal G}$. To do so, we need to express
Assumptions \ref{ass:capacity} and \ref{ass:interpolation} in
${\cal G}$, which we can do using the following linear algebra property.

\begin{lemma}[Linear algebra on compact operators]
 There exist $(u_{i})_{i\in\N}$ an orthogonal basis of ${\cal G}_{\X}$,
 $(v_{i})_{i\in\N}$ an orthogonal basis of $L^2(\X, \R, \rho_{\X})$,
 and $(\lambda_{i})_{i\in \N}$ a decreasing sequence of positive real number such
 that 
 \begin{align}
   S = \sum_{i\in\N} \lambda_{i}^{1/2}u_{i}v_{i}^{\star},\qquad
   S^{\star} = \sum_{i\in\N} \lambda_{i}^{1/2}v_{i}u_{i}^{\star},\qquad
   \Sigma = \sum_{i\in\N} \lambda_{i} u_{i}u_{i}^{\star},\qquad
   K = \sum_{i\in\N} \lambda_{i}v_{i}v_{i}^{\star},
 \end{align}
 where the convergence of series as to be understood with the operator norms.
 Moreover, we have that, if the kernel $k$ is bounded by $\kappa^2$,
 \[
   \sum_{i\in\N} \lambda_{i} \leq \kappa^2 < +\infty.
 \]
 Therefore, $K$ and $\Sigma$ are trace-class, and $S$ and $S^{\star}$ are Hilbert-Schmidt.
\end{lemma}
\begin{proof}
  First of all notice that $\Sigma = \E_{X}[k_{X}\otimes k_{X}]$ and that
  $\norm{k_{x}\otimes k_{x}}_{\op({\cal G}_{\X})} = \norm{k_{x}}_{{\cal G}_{\X}} = k(x, x)
  \leq \kappa^2$. Therefore $\Sigma$ is a nuclear operator, so it is trace class
  and so it is compact.
  
  The first point results from diagonalization of kernel operator, known as
  Mercer's Theorem \citep{Mercer1909,Steinwart2012}.
  $\Sigma$ is a compact operator, therefore, the Spectral Theorem gives the
  existence of a sequence $(\lambda_{i}) \in \R^{\N}$ and a orthonormal basis
  $(u_{i}) \in {\cal G}_{\X}^{\N}$ of ${\cal G}_{\X}$ such that
  \[
    \Sigma = \sum_{i\in\N} \lambda_{i} u_{i} u_{i}^{\star},
  \]
  where the convergence has to be understood with the operator norm.
  Because $\Sigma$ is of the form $S^{\star} S$, one can consider $(\lambda_{i})$ a
  decreasing sequence of positive eigen value.
  Then, by defining, for all $i\in\N$ with $\lambda_{i} > 0$,
  \[
    v_{i} = \lambda_{i}^{-1/2} Su_{i}
  \]
  we check that $(v_{i})$ are orthonormal, and we complete them to form an
  orthonormal basis of $(L^2(\X, \R, \rho_{\X}))$. Finally we check that
  \[
    S = \sum_{i\in\N} \lambda_{i}^{1/2} v_{i}u_{i}^{\star},
  \]
  and that the other equalities hold too.
  
  To check the second assertion, we use that $k_{x}k_{x}^{\star}$ is rank one
  when operating on ${\cal G}_{\X}$ and therefore
  \begin{align*}
    \trace{\Sigma} &= \trace\paren{\E_{X}[k_{X}k_{X}^{\star}]} = \E_{X}\bracket{\trace\paren{k_{X}k_{X}^{\star}}}
    = \E_{X}\bracket{\norm{k_{X}k_{X}^{\star}}_{\op({\cal G}_{\X})}}
    \\&= \E_{X}\bracket{\norm{k_{X}}_{{\cal G}_{\X}}} = \E_{X}[k(x, x)] \leq \kappa^2.
  \end{align*}
  This shows that $S$ and $S^{\star}$ are Hilbert-Schmidt operators and that
  $K$ is also trace class.
\end{proof}

This allow us to cast in ${\cal G}_{\X}$ the assumptions expressed in $L^2$.

\begin{lemma}[Equivalence of capacity condition]
  For $\sigma \in (0, 1]$, it is equivalent to suppose that
  \begin{itemize}
    \item $\trace_{L^2(\X, {\cal H}, \rho_{\X})}(K^\sigma) < +\infty$.
    \item $\trace_{{\cal G}_{\X}}(\Sigma^\sigma) < +\infty$.
    \item $\sum_{i\in\N} \lambda_{i}^\sigma < +\infty$.
  \end{itemize}
\end{lemma}

In Assumption \ref{ass:capacity}, the smaller $\sigma$, the faster the $\lambda_{i}$
decrease, the easier is will be to approximate $\Sigma$ based on approximation
of $\rho$. This appears explicitly in Theorem \ref{thm:matrix}.
Indeed, for $\sigma=0$, the condition should be defined as $\Sigma$ of finite rank. 
Note that when $k$ is bounded, we know that $\Sigma$ is trace class, and
therefore, Assumption \ref{ass:capacity} holds with $\sigma = 1$.

\begin{lemma}[Interpolation inequality in RKHS]
  \label{lem:interpolation}
  Assumption \ref{ass:interpolation} implies that
  \begin{equation}
    \forall\ \gamma\in{\cal G},\qquad
    \norm{S\gamma}_{L^\infty} \leq c_p \norm{\Sigma^{\frac{1}{2}-p}\gamma}_{\cal G}.
  \end{equation}
\end{lemma}
\begin{proof}
  We begin by showing the property for $\gamma \in {\cal G}_{\X}$.
  When $\gamma = \sum_{i\in\N} c_{i} v_{i}$ with $\sum_{i\in\N} c_{i}^2 < +\infty$,
  denote $g = \sum_{i\in\N} \lambda_{i}^{\frac{1}{2}-p}c_{i}u_{i}$, we have $g\in
  L^2$, therefore, using Assumption \ref{ass:interpolation},
  \begin{align*}
    \norm{S\gamma}_{L^\infty}
    = \norm{K^pg}_{L^\infty}
    \leq c_p\norm{g}_{L^2}
    = c_p\norm{\Sigma^{\frac{1}{2}-p}\gamma}_{{\cal G}_{\X}}.
  \end{align*}
  This ends the proof for ${\cal G}_{\X}$.
  Note also that when the result of the Lemma holds, then Assumption
  \ref{ass:interpolation} holds for any $g\in\ima_{L^2(\X, \R,\rho_{\X})}
  K^{\frac{1}{2}-p}$.

  Let switch to ${\cal G}$ now.
  Let $\gamma \in {\cal G}$, and denote $g = S\gamma$. Suppose that $g$ achieve
  it maximum in $x_\infty$, define the direction $\xi =
  \sfrac{g(x_\infty)}{\norm{g(x_\infty)}_{\cal H}}$, and define
  $g_{\xi}: x\to \scap{g(x)}{\xi}_{\cal H} \in L^2(\X, \R, \rho_{\X})$, and
  $\gamma_{\xi} = \sum_{j\in\N} \scap{g_{\xi}}{v_{i}}_{L^2} u_{i} \in {\cal G}_{\X}$.
  We have
  \[
    \norm{S\gamma}_{L^\infty} = \norm{S\gamma_{\xi}}_{L^\infty} \leq
    c_p\norm{\Sigma^{\frac{1}{2}-p}\gamma_{\xi}}_{{\cal G}_{\X}}
    \leq c_p\norm{\Sigma^{\frac{1}{2}-p}\gamma}_{{\cal G}}.
  \]
  When $g$ does not achieve its maximum, one can do a similar reasoning by
  considering a basis $(f_{i})_{i\in\N}$ of ${\cal H}$ and decomposition $\gamma$ on
  the basis $(u_{i}f_{j})_{i,j\in\N}$, before summing the directions.
\end{proof}

In Assumption \ref{ass:interpolation}, the bigger $\sfrac{1}{2}-p$ the more we are able to
control our problem in ${\cal G}$, the better. Note that this reformulation of
the interpolation inequality allow to generalized it for $p$ smaller than zero.
Note that when $k$ is bounded, 
\(
\norm{(S\gamma)(x)}_{\cal H} = \norm{k_{X}^{\star} \gamma}_{\cal H}
\leq \norm{k_{X}}_{\op} \norm{\gamma}_{\cal G} = \sqrt{k(x, x)}\norm{\gamma}_{\cal G},
\)
hence Assumption \ref{ass:interpolation} holds with $p=\sfrac{1}{2}$.

\subsection{Linear algebra with atoms \texorpdfstring{$k_{x}$}{} and useful inequalities}

From the study of the convergence of $g_{n}$ to $g_\lambda$ will emerge two
quantities linked to eigen values of $\Sigma$ and the position of $k_{x}$
regarding eigen spaces, that are
\begin{equation}\label{eq:eigen-quantity}
  {\cal N}(\lambda) = \trace\paren{(\Sigma+\lambda)^{-1}\Sigma},\qquad
  {\cal N}_\infty(\lambda) = \sup_{x\in\supp\rho_{\X}}\norm{(\Sigma+\lambda)^{-\frac{1}{2}}k_{x}}_{\op}.
\end{equation}
While those quantity could be bounded with brute force consideration,
Assumptions \ref{ass:capacity} and \ref{ass:interpolation} will help to control
them more subtly.

\begin{proposition}[Characterization of capacity condition]
  The property $\sum_{i\in\N}\lambda_{i}^\sigma < +\infty$, can be rephrased in
  term of eigen values of $\Sigma$ as the existence of a $a_1 > 0$ such that,
  for all $i > 0$,
  \begin{equation}
    \lambda_{i} \leq a_1 (i+1)^{-\frac{1}{\sigma}}.
  \end{equation}
\end{proposition}
\begin{proof}
  Denote by $u_{i}$ and $S_{n}$ the respective quantities $\lambda_{i}^{\sigma}$ and
  $\sum_{i=1}^{n} u_{i}$. Because $S_{n}$ converge, it is a Cauchy sequence, so there
  exits $N$ such that for any $p > q > N$, \( S_p - S_{q} = \sum_{i=q+1}^p u_{i}
  \leq 1. \) In particular, considering $p=2q$, and because $(\lambda_{i})$ is
  decreasing, we have \( q u_{2q} \leq \sum_{i=q+1}^{2q} u_{i} \leq 1. \)
  Therefore, we have that for all $i > 2N$, $u_{i} \leq 3(i+1)^{-1}$, considering
  $(a_1)^{\sigma} = 3 + \max_{i\leq 2N}\brace{(i+1)u_{i}}$, we get the desired
  result.
\end{proof}

\begin{proposition}[Characterization of ${\cal N}$]
  When $\trace\paren{K^{\sigma}} < +\infty$, with $a_2 = \int_0^\infty
  \frac{a_1}{a_1 + t^{\frac{1}{\sigma}}}\diff t$,
  \begin{equation}
    \forall\ \lambda > 0, \qquad {\cal N}(\lambda, r) \leq a_2 \lambda^{-\sigma}.
  \end{equation}
\end{proposition}
\begin{proof}
  Expressed with eigenvalues, we have
  \[
    {\cal N}(\lambda) = \trace\paren{(\Sigma + \lambda)^{-1}\Sigma} =
    \sum_{i\in\N} \frac{\lambda_{i}}{\lambda_{i} + \lambda}.
  \]
  Using that $\lambda_{i} \leq a_1(i+1)^{-\frac{1}{\sigma}}$, that
  $x\to\frac{x}{x+a}$ is increasing with respect to $x$ for any $a> 0$ and the
  series-integral comparison, we get for $\sigma \in (0, 1]$
  \begin{align*}
    {\cal N}(\lambda) &\leq \sum_{i\in\N} \frac{a_1(i+1)^{-\frac{1}{\sigma}}}{a_1(i+1)^{-\sigma} + \lambda}
                        \leq \int_0^\infty \frac{a_1t^{-\frac{1}{\sigma}}}{a_1t^{-\frac{1}{\sigma}} + \lambda}\diff t
                        = \int_0^\infty \frac{a_1}{a_1 + \lambda t^{\frac{1}{\sigma}}}\diff t
    \\&= \lambda^{-\sigma}\int_0^\infty \frac{a_1}{a_1 + (\lambda^\sigma t)^{\frac{1}{\sigma}}}\diff (\lambda^\sigma t)
    = a_2\lambda^{-\sigma},
  \end{align*}
  where we check the convergence of the integral.
\end{proof}

Indeed, Assumption \ref{ass:interpolation} has a profound linear algebra
meaning, it is a condition on $\rho_{\X}$-almost all the vector $k_{x} \in {\cal
  G}_{\X}$ not to be excessively supported on the eigenvector corresponding to
small eigenvalue of $\Sigma$.

\begin{proposition}[Characterization of interpolation condition]
  The interpolation Assumption \ref{ass:interpolation} implies that,
  for all $i\in\N$
  \begin{equation}
    \sup_{x\in\rho_{\X}} \abs{\scap{k_{x}}{u_{i}}_{{\cal G}_{\X}}} \leq c_p\lambda_{i}^{\frac{1}{2}-p}.
  \end{equation}
\end{proposition}
\begin{proof}
  Consider the decomposition of $k_{x} \in {\cal G}_{\X}$ according to the eigen
  vectors of $\Sigma$, with $a_{i}(x) = \scap{k_{x}}{u_{i}}$. The interpolation
  condition Assumption \ref{ass:interpolation}, expressed in ${\cal G}_{\X}$ with
  Lemma \ref{lem:interpolation}, leads to for any $\gamma_{\X} \in {\cal G}_{\X}$,
  and $S\gamma_{\X} :\X\to\R$,
  \begin{align*}
    \abs{(S\gamma_{\X})(x)}
    = \abs{\scap{k_{x}}{\gamma_{\X}}_{{\cal G}_{\X}}}
    \leq \norm{S\gamma_{\X}}_{L^\infty}
    \leq c_p \norm{\Sigma^{\frac{1}{2}-p} \gamma_{\X}}_{{\cal G}_{\X}}
  \end{align*}
  This implies that
  \[
    \scap{k_{x}}{\gamma_{\X}}_{{\cal G}_{\X}}^2 = \paren{\sum_{i\in\N}
      \scap{k_{x}}{u_{i}}\scap{\gamma_{\X}}{u_{i}}}^2 \leq c_p^2
    \norm{\Sigma^{\frac{1}{2}-p} \gamma_{\X}}_{{\cal G}_{\X}}^2 = c_p^2
    \sum_{i\in\N} \lambda_{i}^{1-2p} \scap{\gamma_{\X}}{u_{i}}^2.
  \]
  Taking $\gamma_{\X} = u_{i}$, we get that
  \[
    \abs{\scap{k_{x}}{u_{i}}} \leq c_p\lambda_{i}^{\frac{1}{2}-p}.
  \]
  This result relates the interpolation condition to the fact that $k_{x}$ is not
  excessively supported on the eigenvectors corresponding to vanishing eigenvalues of $\Sigma$.
\end{proof}

\begin{proposition}[Characterization of ${\cal N}_{\infty}(\lambda, r)$]
  Under the interpolation condition, Assumption \ref{ass:interpolation}, we have
  with $a_3 = c_p (2p)^{-p} (1-2p)^{\frac{1}{2}-p}$, or $a_3=c_p$ when
  $p=\sfrac{1}{2}$,
  \begin{equation}
    {\cal N}_{\infty}(\lambda) \leq a_3\lambda^{-p}.
  \end{equation}
\end{proposition}
\begin{proof}
  First of all, notice that
  \begin{align*}
    \norm{(\Sigma + \lambda)^{-\frac{1}{2}}k_{x}}_{{\cal G}_{\X}}
    &= \sup_{\norm{\gamma_{\X}}_{\cal \X} = 1} \scap{\gamma_{\X}}{(\Sigma + \lambda)^{-\frac{1}{2}}k_{x}}_{{\cal G}_{\X}}
      = \sup_{c; \sum_{i\in\N} c_{i}^2 = 1} \sum_{i\in\N} \frac{c_{i} \scap{k_{x}}{u_{i}}}{(\lambda + \lambda_{i})^{\frac{1}{2}}}
    \\& \leq c_p \sup_{c;\sum_{i\in\N c_{i}^2=1}} \sum_{i\in\N} \frac{c_{i}\lambda_{i}^{\frac{1}{2}-p}}{(\lambda + \lambda_{i})^{\frac{1}{2}}}
    \leq \sup_{t\in\R_+} c_p \frac{t^{\frac{1}{2}-p}}{(\lambda + t)^{\frac{1}{2}}}.
  \end{align*}
  When $p\in(0, \sfrac{1}{2})$, this last function is zero in zero and in
  infinity, therefore its maximum $t_0$ verifies, taking the derivative of its
  logarithm,
  \[
    \frac{\sfrac{1}{2}-p}{t_0} = \frac{1}{2(t_0+\lambda)} \quad\Rightarrow\quad
    t_0 = \frac{(1-2p)\lambda}{2p} \quad\Rightarrow\quad
    \sup_{t\in\R_+}\frac{t^{\frac{1}{2}-p}}{(\lambda + t)^{\frac{1}{2}}} =
    (2p)^{-p} (1-2p)^{\frac{1}{2}-p}\lambda^{-p}.
  \]
  The cases $p\in\brace{0, 1}$ are easy to treat.
\end{proof}

In the previous analysis, one fact does not appear, it is that $\Sigma$ and
$k_{x}$ are linked to one another, since $\Sigma = \E_{X}[k_{X}k_{X}^{\star}]$. The following remark
builds on it to relates $\cal N$ and ${\cal N}_\infty$.

\begin{remark}[Relation between interpolation and capacity condition]
  The capacity and interpolation condition are related by the fact that it
  unreasonable not to consider that
  \(
    p \leq \sfrac{\sigma}{2}.
  \) 
\end{remark}
\begin{proof}
  Because $k_{x}k_{x}^{\star}$ is of rank one in ${\cal G}_{\X}$, we have
  \begin{align*}
    {\cal N}(\lambda) &= \trace\paren{(\Sigma + \lambda)^{-1}\Sigma}
                           = \E_{X}\bracket{\trace\paren{(\Sigma + \lambda)^{-1}k_{X}k_{X}^{\star}}}
                           = \E_{X}\bracket{\trace\paren{k_{X}^{\star}(\Sigma + \lambda)^{-1}k_{X}}}
    \\&= \E_{X}\bracket{\norm{k_{X}^{\star}(\Sigma + \lambda)^{-1}k_{X}}_{\op}}
    = \E_{X}\bracket{\norm{(\Sigma + \lambda)^{-\frac{1}{2}}k_{X}}^2_{{\cal G}_{\X}}}.
  \end{align*}
  So indeed, ${\cal N}(\lambda)$ is the expectation of the square $\norm{(\Sigma
    + \lambda)^{-\frac{1}{2}}k_{X}}_{{\cal G}_{\X}}$, when ${\cal
    N}_{\infty}(\lambda)$ is the supremum of this last quantity. Therefore
  \[
    {\cal N}(\lambda) \leq {\cal N}_{\infty}(\lambda)^2
  \]
  Supposing that the dependency in $\lambda$ proved above are tight, we should
  have \( \sigma \geq 2p, \) which is the statement of
  this remark. We refer the reader to Lemma 6.2. of \citet{Fischer2020} for more
  consideration to relates $\sigma$ and $p$ (reading $p$ and
  $\sfrac{\alpha}{2}$ with their notations)
\end{proof}

\subsection{Geometrical control of the residual \texorpdfstring{$\norm{g_\lambda - g^{*}}_{L^\infty}$}{}}
\label{proof:krr-1}

The proof of the first assertion in Lemma \ref{lem:rkhs} follows from, using
Assumption \ref{ass:source}, with $g_0 \in K^{-q}g^{*}$,
\begin{align*}
  g_\lambda - g^{*}  &= (K(K + \lambda)^{-1} - I)g^{*} = -\lambda (K+\lambda)^{-1} g^{*}
  = -\lambda (K + \lambda)^{-1}K^{q}g_0 \\&= -\lambda K^{p}(K + \lambda)^{-1}K^{q-p}g_0.
\end{align*}
Then using Assumption \ref{ass:interpolation},
\begin{align*}
  \norm{g^{*} - g_\lambda}_{\infty} &\leq c_p \lambda\norm{K^{q-p}(K + \lambda)^{-1}}_{\op}\norm{g_0}_{L^2}
  \\&\leq c_p\lambda \norm{K(K + \lambda)^{-1}}^{q-p}_{\op} \norm{(K + \lambda)^{-1}}^{1+p-q}_{\op}\norm{g_0}_{L^2}
  \\&\leq c_p\lambda 1^{q-p} \lambda^{-(1+p-q)}\norm{g_0}_{L^2}
  = b_1\lambda^{q-p},
\end{align*}
where we have used that $\norm{K(K + \lambda)^{-1}}_{\op} =
\sfrac{\norm{K}_{\op}}{(\norm{K}_{\op} + \lambda)} \leq 1$ and that
$\norm{(K + \lambda)^{-1}} \leq \lambda^{-1}$.

\subsection{Convergence of \texorpdfstring{$\norm{g_{n} - g_\lambda}$}{} through concentration inequality}
\label{proof:krr-2}

For the proof of the second assertion in Lemma \ref{lem:rkhs}, we will put
ourselves in ${\cal G}$. For this, we define in ${\cal G}$
\begin{equation}
  \gamma = \E_\rho[k_{X}\phi(Y)],\qquad
  \gamma_\lambda = (\Sigma + \lambda)^{-1}\gamma,\qquad
  \hat\gamma = \E_{\hat\rho}[k_{X}\phi(Y)],
\end{equation}
so that $g_\lambda = S\gamma_\lambda$, and
$g_{n} = S(\hat\Sigma + \lambda)^{-1}\hat\gamma$.

\subsubsection{Decomposition into a matrix and a vector term}
We begin by expressing $g_{n} - g_\lambda$ in ${\cal G}$ with
\begin{align*}
  g_{n} - g_\lambda 
  &= S\paren{(\hat\Sigma + \lambda)^{-1} \hat\gamma - (\Sigma + \lambda)^{-1} \gamma}
  \\&= S\paren{(\hat\Sigma + \lambda)^{-1} (\hat\gamma - \gamma) + ((\hat\Sigma + \lambda)^{-1} - (\Sigma + \lambda)^{-1}) \gamma)}
  \\&= S\paren{(\hat\Sigma + \lambda)^{-1} (\hat\gamma - \gamma) + (\hat\Sigma + \lambda)^{-1}(\Sigma - \hat\Sigma)(\Sigma + \lambda)^{-1} \gamma)}
  \\&= S\paren{(\hat\Sigma + \lambda)^{-1} ((\hat\gamma - \hat\Sigma \gamma_\lambda) - (\gamma - \Sigma \gamma_\lambda))},
\end{align*}
where we have used that $A^{-1} - B^{-1} = A^{-1}(B - A)B^{-1}$.
Therefore, using the expression, Lemma \ref{lem:interpolation}, of Assumption
\ref{ass:interpolation} in ${\cal G}$, we get
\begin{align*}
  \norm{g_{n} - g_\lambda}_{L^\infty}
  &\leq c_p\norm{\Sigma^{\frac{1}{2}-p} (\hat\Sigma + \lambda)^{-1} (\Sigma + \lambda)^{\frac{1}{2}+p}}_{\op} \times \cdots \\ 
  &\qquad\qquad \norm{(\Sigma + \lambda)^{-(\frac{1}{2}+p)}((\hat\gamma - \hat\Sigma \gamma_\lambda) - (\gamma - \Sigma \gamma_\lambda))}_{\cal G}.
\end{align*}
On the one hand, we have concentration of matrix term towards 
$\Sigma^{\frac{1}{2}-p} (\Sigma + \lambda)^{-\paren{\frac{1}{2}-p}} \preceq I$.
On the other hand, we have concentration of the vector
$\hat\gamma - \hat\Sigma \gamma_\lambda$ towards $\gamma - \Sigma
\gamma_\lambda$.
Indeed the concentration of the matrix term is hard to prove (it is only a
conjecture), therefore we will go for an other decomposition, that will result
in similar rates when $p \geq 0$, that is 
\begin{equation}
  \begin{split}
    &\norm{g_{n} - g_\lambda}_{L^\infty}
    \leq c_p\norm{\Sigma^{\frac{1}{2}-p} (\Sigma + \lambda)^{-\frac{1}{2}}}_{\op} {\cal A}(\lambda){\cal B}(\lambda)\\  
    &{\cal A}(\lambda) = \norm{(\Sigma + \lambda)^{\frac{1}{2}}(\hat\Sigma + \lambda)^{-1} (\Sigma + \lambda)^{\frac{1}{2}}}_{\op}, \\
    &{\cal B}(\lambda) = \norm{(\Sigma + \lambda)^{-\frac{1}{2}}((\hat\gamma - \hat\Sigma \gamma_\lambda) - (\gamma - \Sigma \gamma_\lambda))}_{\cal G}.
  \end{split}
\end{equation}
Recall the definition of the following important quantity that are going to pop up from
the analysis
\begin{equation}\tag{\ref{eq:eigen-quantity}}
  {\cal N}(\lambda) = \trace\paren{(\Sigma+\lambda)^{-1}\Sigma},\qquad
  {\cal N}_\infty(\lambda) = \sup_{x\in\supp\rho_{\X}}\norm{(\Sigma+\lambda)^{-\frac{1}{2}}k_{x}}_{\op}.
\end{equation}

\subsubsection{Extra matrix term}
We control the extra matrix term with
\[
  \norm{\Sigma^{\frac{1}{2}-p} (\Sigma + \lambda)^{-\frac{1}{2}}}_{\op}
  = \norm{\Sigma^{\frac{1}{2}-p} (\Sigma +
    \lambda)^{-\paren{\frac{1}{2}-p}}}_{\op} \norm{(\Sigma + \lambda)^{-p}}_{\op}
  \leq \lambda^{-p}.
\]
Using that $\norm{(\Sigma + \lambda)^{-1}}_{\op} \leq \lambda^{-1}$ and that
$\norm{(\Sigma + \lambda)^{-1}\Sigma}_{\op} \leq
\sfrac{\norm{\Sigma}_{\op}}{(\norm{\Sigma}_{\op} + \lambda)} \leq 1$.

\subsubsection{Matrix concentration}
Let us make explicit the concentration in the matrix term with
\begin{align*}
  (\Sigma + \lambda)^{\frac{1}{2}} (\hat\Sigma + \lambda)^{-1} (\Sigma + \lambda)^{\frac{1}{2}}
  &=  I +
    (\Sigma + \lambda)^{\frac{1}{2}} \paren{(\hat\Sigma + \lambda)^{-1} - (\Sigma + \lambda)^{-1}} (\Sigma + \lambda)^{\frac{1}{2}}
  \\&= I +
    (\Sigma + \lambda)^{\frac{1}{2}} (\hat\Sigma + \lambda)^{-1} \paren{\Sigma - \hat\Sigma} (\Sigma + \lambda)^{-1} (\Sigma + \lambda)^{\frac{1}{2}}.
\end{align*}
From here, notice the following implications (that are actually equivalence)
\begin{align*}
  \Sigma - \hat\Sigma \preceq t (\Sigma + \lambda)
  &\quad\Rightarrow\quad \hat\Sigma + \lambda \succeq (1 - t)(\Sigma + \lambda)
  \\&\quad\Rightarrow\quad (\hat\Sigma + \lambda)^{-1} \preceq (1 - t)^{-1}(\Sigma + \lambda)^{-1}.
  \\&\quad\Rightarrow\quad (\hat\Sigma + \lambda)^{-1}  - (\Sigma + \lambda)^{-1} \preceq t(1 - t)^{-1}(\Sigma + \lambda)^{-1}.
  \\&\quad\Rightarrow\quad (\Sigma + \lambda)^{\frac{1}{2}}
  \paren{(\hat\Sigma + \lambda)^{-1} - (\Sigma + \lambda)^{-1}} (\Sigma + \lambda)^{\frac{1}{2}}\preceq t(1-t)^{-1}
  \\&\quad\Rightarrow\quad (\Sigma + \lambda)^{\frac{1}{2}}
  (\hat\Sigma + \lambda)^{-1} (\Sigma + \lambda)^{\frac{1}{2}} \preceq (1-t)^{-1}.
\end{align*}
The probability of the event $\Sigma - \hat\Sigma \preceq t (\Sigma + \lambda)$,
can be studied through the probability of the event
$(\Sigma + \lambda)^{-\frac{1}{2}}(\Sigma - \hat\Sigma)(\Sigma + \lambda)^{-\frac{1}{2}} \preceq t$, 
which can be studied through concentration of self adjoint operators. Finally,
we have shown that
\begin{equation}
  \norm{(\Sigma + \lambda)^{-\frac{1}{2}}(\Sigma - \hat\Sigma)(\Sigma + \lambda)^{-\frac{1}{2}}}_{\op} \leq t
  \quad\Rightarrow\quad  {\cal A}(\lambda)
  \leq \frac{1}{1-t}.
\end{equation}
The best result that we are aware of, for covariance matrix inequality, is the
extension to self-adjoint Hilbert-Schmidt operators provided by
\citet{Minsker2017} in Section 3.2 of its concentration inequality on random
matrices Theorem 3.1. It can be formulated as the following.

\begin{theorem}[Concentration of self-adjoint operators \citep{Minsker2017}]
  \label{thm:matrix}
  Let denote by $(\xi_{i})_{i\leq n}$ a sequence of independent self-adjoint operator acting on an
  separable Hilbert space ${\cal A}$, such that $\ker(\E[\xi_{i}]) = {\cal A}$,
  that are bounded by a constant $M \in \R$, in the sense
  $\norm{\xi_{i}}_{\op} \leq M$, with finite variance
  $\sigma^2 = \norm{\E\sum_{i=1}^{n} \xi_{i}^2}_{\op}$.
  For any $t>0$ such that $6 t^2 \geq (\sigma^2 + \sfrac{Mt}{3})$,
  \[
    \Pbb\paren{\norm{\sum_{i=1}^{n}\xi_{i}}_{\op} > t} \leq 14\ 
    r\paren{\sum_{i=1}^{n} \E\xi_{i}^2}
    \exp\paren{-\frac{t^2}{2\sigma^2 + 2t M / 3}},
  \]
  with $r(\xi) = \sfrac{\trace{\xi}}{\norm{\xi}_{\op}}$.
\end{theorem}

Let us define $\xi$ that goes from $\X$ to the space of self-adjoint operator
action on ${\cal G}_{\X}$ as
\begin{equation}
  \xi(x) = (\Sigma + \lambda)^{-\frac{1}{2}} k_{x}k_{x}^{\star} (\Sigma + \lambda)^{-\frac{1}{2}}.
\end{equation}
We have that $(\Sigma + \lambda)^{-\frac{1}{2}}(\Sigma - \hat\Sigma)(\Sigma +
\lambda)^{-\frac{1}{2}} = \E_\rho[\xi(X)] - \frac{1}{n}\sum_{i=1}^{n} \xi(x_{i})$.
To apply operator concentration, we need to bound $\xi$ and its variance.

\paragraph{Bound on $\xi$.}
To bound $\xi$ we proceed with, because $k_{x}k_{x}^{\star}$ is of rank one,
\begin{align*}
  \norm{\xi(x)}_{\op} &= \norm{ (\Sigma + \lambda)^{-\frac{1}{2}} k_{x}k_{x}^{\star} (\Sigma + \lambda)^{-\frac{1}{2}}}_{\op}
                        = \trace\paren{ (\Sigma + \lambda)^{-\frac{1}{2}} k_{x}k_{x}^{\star} (\Sigma + \lambda)^{-\frac{1}{2}}}
  \\&= \trace\paren{ k_{x}^{\star} (\Sigma + \lambda)^{-1} k_{x}}
  = \norm{(\Sigma + \lambda)^{-\frac{1}{2}}k_{x}}_{{\cal G}_{\X}}^2
  \leq {\cal N}_{\infty}(\lambda)^2.
\end{align*}

\paragraph{Variance of $\xi$.}
For the variance of $\xi$ we proceed by noticing that
\begin{align*}
  \E\xi(X) &=  \E_{X}(\Sigma + \lambda)^{-\frac{1}{2}} k_{X}k_{X}^{\star} (\Sigma + \lambda)^{-\frac{1}{2}}
             = (\Sigma + \lambda)^{-\frac{1}{2}} \E_{X}\bracket{k_{X}k_{X}^{\star}} (\Sigma + \lambda)^{-\frac{1}{2}}
  \\&= (\Sigma + \lambda)^{-\frac{1}{2}} \Sigma (\Sigma + \lambda)^{-\frac{1}{2}}
             = (\Sigma + \lambda)^{-1} \Sigma.
\end{align*}
Hence
\begin{align*}
  \E\xi(X)^2 \preceq \sup_{x\in\X} \norm{\xi(x)}_{\op} \E[\xi(X)] \preceq {\cal N}_\infty(\lambda)^2(\Sigma + \lambda)^{-1} \Sigma.
\end{align*}
And as a consequence
\begin{equation*}
  \norm{\E\xi(x)^2} \leq {\cal N}_\infty(\lambda)^2,
\end{equation*}
where we have used that $\norm{(\Sigma + \lambda)^{-1} \Sigma}_{\op} =
\sfrac{\norm{\Sigma}_{\op}}{(\norm{\Sigma}_{\op} + \lambda)} \leq 1$.

\paragraph{Concentration bound on $\xi$.}
Using the self-adjoint concentration theorem, we get for any $t>0$, such that
$6nt^2 \geq {\cal N}_{\infty}(\lambda)^2 (1 + \sfrac{t}{3})$,
\begin{align*}
  \Pbb_{{\cal D}_{n}}\paren{\norm{\E_{\hat\rho}[\xi] - \E_\rho[\xi]}_{\op} > t}
  \leq 14\, \frac{\norm{\Sigma}_{\op} + \lambda}{\norm{\Sigma}_{\op}}
  {\cal N}(\lambda)
  \exp\paren{-\frac{nt^2}{2{\cal N}_{\infty}(\lambda)^2 (1 + t/3)}}.
\end{align*}
Therefore, using the contraposition of the prior implication, we get
\begin{equation}
  \Pbb_{{\cal D}_{n}}\paren{{\cal A}(\lambda) > \frac{1}{1-t}} \leq
  14\, \frac{\norm{\Sigma}_{\op} + \lambda}{\norm{\Sigma}_{\op}}
  {\cal N}(\lambda)
  \exp\paren{-\frac{nt^2}{2{\cal N}_{\infty}(\lambda)^2 (1 + t/3)}}.
\end{equation}

\subsubsection{Decomposition of vector term in a variance and a bias term}
Let switch to the vector term, consider $\xi:\X\times\Y\to{\cal G}$, defined as
\begin{equation*}
  \xi = (\Sigma + \lambda)^{-\frac{1}{2}}k_{x}(\phi(y) - k_{x}^{\star} \gamma_\lambda).
\end{equation*}
It allows to express in simple form the vector term as
\[
  {\cal B}(\lambda) =
  \norm{\frac{1}{n}\sum_{i=1}^{n} \xi(X_{i}, Y_{i}) - \E_{(X, Y)\sim\rho}[\xi(X, Y)]}.
\]
We can study this term through concentration inequality in ${\cal G}$.
To proceed we will dissociate the variability due to $Y$
to the one due to $X$, recalling that $g_\lambda(x) = k_{x}^{\star}\gamma_\lambda$
and going for the following decomposition
\begin{equation}
  \begin{split}
    &\xi(x, y) = \xi_v(x, y) + \xi_b(x)\\
    &\xi_v(x, y) = (\Sigma + \lambda)^{-\frac{1}{2}}k_{x}(\phi(y) - g^{*}(x)),\\
    &\xi_b(x) = (\Sigma + \lambda)^{-\frac{1}{2}}k_{x}(g^{*}(x) - g_\lambda(x)),
  \end{split}
\end{equation}
which corresponds to the decomposition
\begin{equation}
  \begin{split}
    &{\cal B}(\lambda) \leq {\cal B}_{v}(\lambda) + {\cal B}_{b}(\lambda)\\
    &{\cal B}_v(\lambda) = \norm{\E_{\hat\rho}[\xi_v(X, Y)]-\E_{\rho}[\xi_v(X, Y)]}\\
    &{\cal B}_b(\lambda) = \norm{\E_{\hat\rho}[\xi_b(X, Y)]-\E_{\rho}[\xi_b(X, Y)]}.
  \end{split}
\end{equation}
The first term is due to the error because of having observed $\phi(y)$ rather
than $g^{*}(x)$, often called ``variance'', and the second term is due to the aiming
for $g_\lambda$ instead of $g^{*}$ often called ``bias''.

\subsubsection{Control of the variance}
To control the variance term, we will use the Bernstein inequality stated Theorem
\ref{thm:bernstein-vector-full}.

\paragraph{Bound on the moment of $\xi_v$.}
First of all notice that
\[
  \norm{\xi_v(x, y)}_{\cal G} \leq \norm{(\Sigma+\lambda)^{-\frac{1}{2}}k_{x}}_{\op} \norm{\phi(y) -
    g^{*}(x)}_{\cal H}.
\]
Therefore, under Assumption \ref{ass:moment}, for $m \geq 2$:
\begin{align*}
  \E_{(X, Y)\sim \rho}\bracket{\norm{\xi_v(X, Y)}^m}
  &\leq
  \E_{X\sim\rho_{\X}}\bracket{\norm{(\Sigma+\lambda)^{-\frac{1}{2}}k_{x}}_{\op}^m
  \E_{Y\sim\rho\vert_{X}}\bracket{\norm{\phi(y) -
  g^{*}(x)}_{\cal H}^m}}
  \\& \leq \frac{1}{2} m! \sigma^2 M^{m-2}\E_{X\sim\rho_{\X}}\bracket{\norm{(\Sigma+\lambda)^{-\frac{1}{2}}k_{x}}_{\op}^m}.
\end{align*}
We bound the last term with
\begin{align*}
  \E_{X\sim\rho_{\X}}\bracket{\norm{(\Sigma+\lambda)^{-\frac{1}{2}}k_{x}}_{\op}^m}
  &\leq \sup_{x\in\supp\rho_{\X}} \norm{(\Sigma+\lambda)^{-\frac{1}{2}}k_{x}}_{\op}^{m-2}
  \E_{X\sim\rho_{\X}}\bracket{\norm{(\Sigma+\lambda)^{-\frac{1}{2}}k_{x}}_{\op}^2}
  \\&= {\cal N}_{\infty}(\lambda)^{(m-2)} {\cal N}(\lambda).
\end{align*}

\paragraph{Concentration on $\xi_v$.}
Applying Theorem \ref{thm:bernstein-vector-full}, we get, for any $t > 0$, that
\begin{equation}
  \Pbb\paren{{\cal B}_v(\lambda) > t}
  \leq 2\exp\paren{-\frac{nt^2}{2\sigma^2{\cal N}(\lambda) + 2 M{\cal N}_\infty(\lambda)t}}.
\end{equation}

\subsubsection{Control of the bias}

To control the bias, we recall a simpler version of Bernstein concentration
inequality, that is a corollary of Theorem \ref{thm:bernstein-vector-full}.

\begin{theorem}[Concentration in Hilbert space \citep{Pinelis1986}]
  \label{thm:bernstein-vector}
  Let denote by ${\cal A}$ a Hilbert space and by $(\xi_{i})$ a sequence of independent
  random vectors on ${\cal A}$ such that $\E[\xi_{i}] = 0$, that are bounded by a
  constant $M$, with finite variance
  $\sigma^2 = \E[\sum_{i=1}^{n}\norm{\xi_{i}}^2]$.
  For any $t>0$,
  \[
    \Pbb(\norm{\sum_{i=1}^{n} \xi_{i}} \geq t) \leq 2\exp\paren{-\frac{t^2}{2\sigma^2 +
        2tM / 3}}.
  \]
\end{theorem}

\paragraph{Bound on $\xi_b$.}
We have
\[
  \norm{\xi_b(x)}_{\cal G} \leq \sup_{x\in\supp\rho_{\X}} \norm{(\Sigma+\lambda)^{-\frac{1}{2}}k_{x}}_{\op}
  \norm{g_\lambda(x) - g^{*}(x)}_{\cal H}
  \leq {\cal N}_{\infty}(\lambda)\norm{g_\lambda - g^{*}}_{\infty}. 
\]
Therefore, with Appendix \ref{proof:krr-1}, we get
\[
  \norm{\xi_b(x)}_{\cal G} \leq b_1 \lambda^{q-p} {\cal N}_{\infty}(\lambda).
\]

\paragraph{Variance of $\xi_b$.}
For the variance we proceed with 
\[
  \norm{\xi_b(x)}_{\cal G}^2 \leq {\cal N}_{\infty}(\lambda)^2
  \norm{g_\lambda(x) - g^{*}(x)}_{\cal H}^2.
\]
Therefore
\[
  \E[\norm{\xi_b(X)}^2] \leq {\cal N}_{\infty}(\lambda)^2 \norm{g_\lambda -
    g^{*}}_{L^2}^2.
\]
Using the derivations made in Appendix \ref{proof:krr-1}, we have, using that $q \leq 1$,
\begin{align*}
  \norm{g_\lambda - g^{*}}_{L^2} &= \lambda \norm{(K+\lambda)^{-1}K^q g_0}_{L^2}
  \leq \lambda \norm{(K+\lambda)^{-(1-q)}}_{\op}
  \norm{(K+\lambda)^{-q}K^q}_{\op}\norm{g_0}_{L^2}
  \\&\leq \lambda^q \norm{g_0}_{L^2}.
\end{align*}

\paragraph{Concentration on $\xi_b$.}
Adding everything together, we get
\begin{equation}
  \Pbb\paren{{\cal B}_b(\lambda) > t}
  \leq 2\exp\paren{-\frac{nt^2}{2
      \paren{\lambda^{2q} {\cal N}_{\infty}(\lambda)^{2}\norm{g_0}^2_{L^2} + b_1\lambda^{q-p}{\cal N}_{\infty}(\lambda)t / 3}}}.
\end{equation}
Note that based on the bound on the variance, we would like
${\cal N}_{\infty}(\lambda)^2\lambda^{2q} \approx\lambda^{2(q-p)}$ to be smaller
than ${\cal N}(\lambda) \approx \lambda^{-\sigma}$. It is the case since $q > p$.

\subsubsection{Union bound}
To control $\norm{g_{n} - g_\lambda}_{L^\infty} \leq
c_p\lambda^{-p} {\cal A}(\lambda) ({\cal B}_v(\lambda) + {\cal B}_b(\lambda))$,
we need to perform a union bound on the control of ${\cal A}$ and the control of
${\cal B} := {\cal B}_v + {\cal B}_b$,
we use that for any $t > 0$ and $0<s<1$, $c_p \lambda^{-p} {\cal A}{\cal B} > t$
implies ${\cal A} > \sfrac{1}{(1-s)}$ or
${\cal B} > \sfrac{(1-s)t \lambda^p}{c_p}$.
Similarly ${\cal B}_v + {\cal B}_b > t$, implies that either ${\cal B}_v >
\sfrac{t}{2}$, either ${\cal B}_b > \sfrac{t}{2}$. Therefore, we have, the
following inclusion of events (with respect to ${\cal D}_{n}$)
\[
  \brace{\norm{g_{n} - g_\lambda}_{L^\infty} > t} \subset \brace{{\cal A} >
    \frac{1}{1-s}}
  \cup \brace{{\cal B}_v > \frac{(1-s)t\lambda^{p}}{2c_p}}
  \cup \brace{{\cal B}_b > \frac{(1-s)t\lambda^{p}}{2c_p}}.
\]
In term of probability this leads to
\begin{equation}
  \Pbb_{{\cal D}_{n}}\paren{\norm{g_{n} - g_\lambda}_{L^\infty} > t} \leq
  \Pbb_{{\cal D}_{n}}\paren{{\cal A} > \frac{1}{1-s}}
  + \Pbb_{{\cal D}_{n}}\paren{{\cal B} > \frac{(1-s)t\lambda^p}{c_p}}.
\end{equation}
Looking closer it is the term in ${\cal B}$ that will be the more problematic,
therefore we would like $s$ to be small. It we take $s$ to be a constant with
respect to $t$, we will get something that behaves like $\Pbb(B > t\lambda^p)$, which is
the best we can hope for (this also explain why we divide ${\cal B} > t$ in
${\cal B}_v > \sfrac{t}{2}$ or ${\cal B}_b > \sfrac{t}{2}$). We will consider
$s=\sfrac{1}{2}$. We express concentration based on the expression of ${\cal N}$
and ${\cal N}_{\infty}$, assuming $\lambda \leq \norm{\Sigma}_{\op}$, and $n > a_3^2\lambda^{-2p}$
\[
  \Pbb_{{\cal D}_{n}}({\cal A} > 2) \leq 28 a_2 \lambda^{-\sigma}\exp(-
  \frac{n\lambda^{2p}}{10 a_3^2}).
\]
Similarly we get, when $\lambda \leq 1$, using that $\lambda^{-\sigma} \geq 1$
\[
  \Pbb_{{\cal D}_{n}}({\cal B}_v > \sfrac{t}{4}) \leq
  2 \exp\paren{-\frac{n\lambda^\sigma t^2}{32\sigma^2 a_2 + 8Ma_3\lambda^{-p}t}}.
\]
For the bias term, we can proceed at a brutal bounding, based on the fact that
for $\lambda \leq 1$, $\lambda^{q-p} \leq 1 \leq \lambda^{-\sigma}$, to get
\[
  \Pbb_{{\cal D}_{n}}({\cal B}_b > \sfrac{t}{4}) \leq
  2 \exp\paren{-\frac{n\lambda^\sigma t^2}{32 a_3^2\norm{g_0}_{L^2} +
      8b_1a_3\lambda^{-p}t/3}}.
\]
With $b_4 = \max(32\sigma^2 a_2, 32 a_3^2\norm{g_0}_{L^2})$ and
$b_5 = \max(8Ma_3, 8b_1a_3/3)$, we get the following union bound
\[
  \Pbb_{{\cal D}_{n}}\paren{{\cal B} > \frac{t\lambda^p}{2}}
  \leq 4 \exp\paren{-\frac{n\lambda^{2p+\sigma}t^2}{b_4 + b_5t}}.
\]
We proceed with the union bound on $\norm{g_{n} - g_\lambda}_{L^\infty}$ as
\[
  \Pbb_{{\cal D}_{n}}(\norm{g_{n} - g_\lambda}_{L^\infty} > t)
  \leq b_2 \lambda^{-\sigma}\exp(-b_3n\lambda^{2p})
  +  4 \exp\paren{-\frac{n\lambda^{2p+\sigma}t^2}{b_4 + b_5t}},
\]
with $b_2 = 28a_2$ and $b_3^{-1} = 10a_3^2$, as long as $b_3n > \lambda^{-2p}$,
and $\lambda \leq \max(1, \norm{K}_{\op})$.

\subsubsection{Refinement of Lemma \ref{lem:rkhs}}

Remark that the uniform control in Lemma \ref{lem:rkhs} is more than we need,
we only need control for each $x$ as described in Assumption \ref{ass:concentration}.
Indeed, if $p(x)$ is such that there exists a constant $\tilde{c_p}$ (that does
not depend on $x$ or $\lambda$), such that for any $i\in\N$
\[
  \scap{k_{x}}{u_{i}}_{{\cal G}_{\X}} \leq \tilde{c_p}\lambda_{i}^{p(x)},
\]
then considering that
\[
  g_{n}(x) - g_\lambda(x) = k_{x}^{\star}(\gamma_{n} - \gamma_\lambda)
  = k_{x}^{\star} \paren{\Sigma+\lambda}^{-\frac{1}{2}} \paren{\Sigma+\lambda}^{\frac{1}{2}}\paren{\gamma_{n} - \gamma_\lambda},
\]
we can get improve the results of Lemma \ref{lem:rkhs} by replacing $p$ by $p(x)$.
While we considered $p = {\sup_{x\in\rho_{\X}} p(x)}$  as a consequence of our
proof scheme, one can expect to end up with the  $\E_{X}[\lambda^{p(X)}]$ instead
of $\lambda^p$ when deriving the proof of Theorems \ref{thm:krr-no-density} and
\ref{thm:krr-low-density} (for which one has to refine Theorem
Theorem \ref{thm:low-density} in order to integrate dependency of $L$ to $x$,
similarly to what is done in Lemma \ref{lem:ref-no-density}),
which will lead to better rates.
Yet, because of complexity of expressing a quantity of the type
$\E_{X}[\phi(p(X))]$, for some function $\phi$, we decided not to present this
improved version in the paper.

\subsection{Proof of Theorem \ref{thm:krr-no-density}}
\label{proof:krr-no-density}

Based on on the proof of Theorem \ref{thm:no-density}, we know that
\[
  \E_{{\cal D}_{n}} {\cal R}(f_{n}) - {\cal R}(f^{*}) \leq \ell_\infty
  \Pbb_{{\cal D}_{n}}\paren{\norm{g_{n} - g^{*}}_{\infty} > t_0}.
\]
Now we use that
\[
  \Pbb_{{\cal D}_{n}}\paren{\norm{g_{n} - g^{*}}_{\infty} > t_0}
  \leq \Pbb_{{\cal D}_{n}}\paren{\norm{g_{n} - g_\lambda}_{\infty} > t_0 - \norm{g_\lambda - g^{*}}_{\infty}}.
\]
The result follows from derivations in Appendix \ref{proof:krr-2}, where we used
that when $k$ is bounded, Assumptions \ref{ass:capacity} and
\ref{ass:interpolation} are verified with $\sigma=1$ and $p=\sfrac{1}{2}$.
Note that we do not need the source assumption, since we can bound directly
$\norm{g_\lambda - g^{*}}_{L^2} \leq \norm{g_\lambda - g^{*}}_{L^\infty} < t_0$
while retaking the proof in Appendix \ref{proof:krr-2}.
Moreover, the results of this last proof holds under the condition $n \lambda
b_3 > 1$, but, since $\E_{{\cal D}_{n}} {\cal R}(f_{n}) - {\cal R}(f^{*}) \leq
\ell_\infty$, we can augment the constant $b_6$ so that the result
in Theorem \ref{thm:krr-no-density} still holds for any $n \in \N^{*}$.

\subsection{Proof of Theorem \ref{thm:krr-low-density}}
\label{proof:krr-low-density}

We can rephrase Lemma \ref{lem:rkhs}, using a union bound
\begin{align*}
  \Pbb_{{\cal D}_{n}}(\norm{g_{n} - g^{*}} > t)
  &\leq \Pbb_{{\cal D}_{n}}(\norm{g_{n} - g_\lambda} > \sfrac{t}{2})
    + \Pbb_{{\cal D}_{n}}(\norm{g_\lambda - g^{*}} > \sfrac{t}{2})
  \\& \leq b_2\lambda^{-\sigma}\exp\paren{-b_3n\lambda^{2p}}
  + 4\exp\paren{-\frac{n\lambda^{2p+\sigma}t^2}{4b_4 + 2b_5t}}
  + \ind{t\leq 2\lambda^{q-p}}.
\end{align*}
Using variant of Theorem \ref{thm:low-density} presented in Appendix 
\ref{app:ref-low-density}, we get
\begin{align*}
  {\cal R}(f_{n}) - {\cal R}(f^{*}) 
  &\leq \ell_\infty b_2 \lambda^{-\sigma} \exp\paren{-b_3n\lambda^{2p}}
  + 2c_{\psi}c_{\alpha} 2^{\alpha + 1} \lambda^{(q-p)(\alpha+1)}
  \\&\qquad\qquad+ 2c_{\psi}c_{\alpha} c\paren{b_4^{\frac{\alpha+1}{2}} (n\lambda^{2p+\sigma})^{-\frac{\alpha+1}{2}} + b_5^{\alpha+1}(n\lambda^{2p+\sigma})^{-(\alpha + 1)}}.
\end{align*}
As long as $\lambda \leq \max(\norm{K}_{\op}, 1)$ and $n \geq
(b_3\lambda^{2p})^{-1}$.
We optimize those rates with $\lambda = \lambda_0 n^{-\gamma}$, and $\gamma$
satisfying
\[
  2\gamma(q-p) = 1 - \gamma(2p+\sigma) \qquad\Rightarrow\qquad
  \gamma = (2q + \sigma)^{-1}.
\]
This leads to, for $n$ after a certain $N\in\N^{*}$
\begin{align*}
  {\cal R}(f_{n}) - {\cal R}(f^{*}) 
  &\leq \ell_\infty b_2 \lambda_0^{-\sigma} n^{\frac{\sigma}{2q+\sigma}} \exp\paren{-b_3n\lambda_0^{2p} n^{\frac{2(q-p) + \sigma}{2q+\sigma}}}
   \\&\qquad + 2c_{\psi}c_{\alpha} 2^{\alpha + 1} \lambda_0^{(q-p)(\alpha+1)} n^{-\frac{(q-p)(\alpha+1)}{2q+\sigma}}
  \\&\qquad + 2c_{\psi}c_{\alpha} c\paren{b_4^{\frac{\alpha+1}{2}} \lambda_0^{\frac{(2p+\sigma)\alpha+1}{2}}  n^{-\frac{(q-p)(\alpha+1)}{2q+\sigma}} + b_5^{\alpha+1} \lambda_0^{(2p+\sigma)\alpha+1}  n^{-\frac{2(q-p)(\alpha+1)}{2q+\sigma}}}
  \\&\leq b_8 n^{-\frac{2(q-p)(\alpha+1)}{2q+\sigma}}.
\end{align*}
Since $\ell$ is bounded, ${\cal R}(f_{n}) - {\cal R}(f^{*}) \leq \ell_\infty$, and
we can always higher $b_8$, in order to have the inequality for any $n\in\N^{*}$.


\end{document}